\newcommand{\X}{\ensuremath{{\mathbf{X}}}}

\newcommand{\Y}{\ensuremath{{Y}}}
\newcommand{\e}{\ensuremath{{\epsilon}}}

\newcommand{\minbeta}{\ensuremath{{M(\truebeta)}}}
\newcommand{\Xa}{\ensuremath{{X(S)}}}
\newcommand{\Xb}{\ensuremath{{X(S^c)}}}

\newcommand{\truebeta}{\ensuremath{{\beta^*}}}
\newcommand{\truebetaj}{\ensuremath{{\beta_j^*}}}

\newcommand{\bet}{\ensuremath{{\beta}}}
\newcommand{\betA}{ {\truebeta(S)}}

\newcommand{\hbet}{\ensuremath{\hat \beta(\lam)}}

\newcommand{\hbetA}{\ensuremath{\hat{\beta}^{(1)}}}

\newcommand{\q}{\ensuremath{{q}}}
\newcommand{\p}{\ensuremath{{p}}}

\newcommand{\lam}{\ensuremath{{\lambda}}}

\newcommand{\sign}{\ensuremath{{\mbox{sign}}}}
\newcommand{\const}{{\sigma^2}}

\def\ind{\begin{picture}(9,8)
         \put(0,0){\line(1,0){9}}
         \put(3,0){\line(0,1){8}}
         \put(6,0){\line(0,1){8}}
         \end{picture}
        }

\newcommand{\SXaXa}{\ensuremath{{\Sigma_{11}}}}
\newcommand{\SXbXb}{\ensuremath{{\Sigma_{22}}}}
\newcommand{\SXbXa}{\ensuremath{{\Sigma_{21}}}}
\newcommand{\SXaXb}{\ensuremath{{\Sigma_{12}}}}

\newcommand{\Example}[1]{\vskip 0.1in\textbf{Example #1}}

\documentclass[10pt]{book}
\usepackage{array}

\RequirePackage[OT1]{fontenc} \RequirePackage{amsmath,natbib}
\RequirePackage[colorlinks,citecolor=blue,urlcolor=blue]{hyperref}
\RequirePackage{hypernat}
\usepackage{enumerate}
\usepackage{amsfonts}
\usepackage{graphicx}
\usepackage{verbatim}
\usepackage{amsthm}

\newtheorem{definition}{Definition}
\newtheorem{theorem}{Theorem}
\newtheorem{corollary}{Corollary}
\newtheorem{lemma}{Lemma}

\begin{document}

%%%%%%%%%%%%%%%%%%%%%%%%%%%%%%%%%%%%%%%%%%%%%%%%%%%%%%%%%%%%%%%%%%%%%%%%%%%%%%%%%%%%%%%%%%%%%%%%%%%%%%%%%%%%%%%%%%%%%%%%%%%%
%%%%%%%%%%%%%%%%%%%%%%%%%%%%%%%%%%%%%%%%%%%%%%%%%%%%%%%%%%%%%%%%%%%%%%%%%%%%%%%%%%%%%%%%%%%%%%%%%%%%%%%%%%%%%%%%%%%%%%%%%%%%

\renewcommand{\baselinestretch}{1.2}

\markright{
\hbox{\footnotesize\rm Preprint}\hfill
}

\markboth{\hfill{\footnotesize\rm JINZHU JIA, KARL ROHE AND BIN YU} \hfill}
{\hfill {\footnotesize\rm THE LASSO UNDER HETEROSCEDASTICITY} \hfill}

\renewcommand{\thefootnote}{}
$\ $\par

%%%%%%%%%%%%%%%%%%%%%%%%%%%%%%%%%%%%%%%%%%%%%%%%%%%%%%%%%%%%%%%%%%%%%%%%%%%%%%%%%%%%%%%%%%%%%%%%%%%%%%%%%%%%%%%%%%%%%%%%%%%%

\fontsize{10.95}{14pt plus.8pt minus .6pt}\selectfont
\vspace{0.8pc}
\centerline{\large\bf The Lasso under
Heteroscedasticity}
%\vspace{2pt}
%\centerline{\large\bf IF A SECOND LINE IS NEEDED}
\vspace{.4cm}
\centerline{Jinzhu Jia$^1$, Karl Rohe$^1$ and Bin Yu$^{1,2}$}
\vspace{.4cm}
\centerline{\it Department of Statistics $^1$ and Department of EECS $^2$}
\centerline{\it University of California, Berkeley}
\vspace{.55cm}
\fontsize{9}{11.5pt plus.8pt minus .6pt}\selectfont

%%%%%%%%%%%%%%%%%%%%%%%%%%%%%%%%%%%%%%%%%%%%%%%%%%%%%%%%%%%%%%%%%%%%%%%%%%%%%%%%%%%%%%%%%%%%%%%%%%%%%%%%%%%%%%%%%%%%%%%%%%%%

\begin{quotation}
\noindent {\it Abstract:}
The performance of the Lasso is well understood under the assumptions of the standard linear model with homoscedastic noise.  However, in several applications, the standard model does not describe the important features of the data.  This paper examines how the Lasso performs on a  non-standard model that is motivated by medical imaging applications. In these applications, the variance of the noise scales linearly with the expectation of the observation.  Like all heteroscedastic models, the noise terms in this Poisson-like model are \textit{not} independent of the design matrix.

More specifically, this paper studies the sign consistency of the Lasso under a sparse Poisson-like model. In addition to studying sufficient conditions for the sign consistency of the Lasso estimate, this paper also gives necessary conditions for sign consistency.
Both sets of conditions are comparable to results for the homoscedastic model, showing that when a measure of the signal to noise ratio is large, the Lasso performs well on both Poisson-like data and homoscedastic data.

Simulations reveal that the Lasso performs equally well in terms of  model selection performance on both Poisson-like data and homoscedastic data (with properly scaled noise variance), across a range of parameterizations. Taken as a whole, these results suggest that the Lasso is robust to the Poisson-like heteroscedastic noise.

\par

\vspace{9pt}
\par
\end{quotation}\par

%%%%%%%%%%%%%%%%%%%%%%%%%%%%%%%%%%%%%%%%%%%%%%%%%%%%%%%%%%%%%%%%%%%%%%%%%%%%%%%%%%%%%%%%%%%%%%%%%%%%%%%%%%%%%%%%%%%%%%%%%%%%

\fontsize{10.95}{14pt plus.8pt minus .6pt}\selectfont

\section{Introduction}
\label{intro}

The Lasso \citep{Tibshirani1996} is widely used in high
dimensional regression for variable selection. Its model selection
performance has been well studied under a standard sparse and homoskedastic
regression model. Several researchers have shown that under
sparsity and regularity conditions, the Lasso can select the true
model asymptotically even when $p\gg n$ \citep{DonohoET2006, MeinshausenB2006, Tropp2006,ZhaoY2006, Wainwright2009}.

To define the Lasso estimate, suppose the observed data  are independent pairs
$\{( x_i, \Y_i)\} \in R^p \times R$ for $i = 1, 2,  \dots , n$
following the linear regression model
\begin{equation}
\Y_i = x_i^T\truebeta + \e_i , \label{linear}
\end{equation}
where $x_i^T$ is a row vector representing the predictors for the
$i$th observation, $Y_i$ is the corresponding $i$th response
variable, $\e_i$'s are independent and mean zero noise terms, and $\truebeta \in R^p$. Use $\X \in R^{n\times p}$ to denote the $n\times p$ design matrix with
$x_k^T=\left(\X_{k1},\ldots,\X_{kp}\right)$ as its $k$th row and
with $X_j=\left(\X_{j1},\ldots,\X_{jn}\right)^T$ as its $j$th
column, then
\[ \X= \left( \begin{array}{c}
x_1^T \\
x_2^T\\
\vdots \\
x_n^T
\end{array} \right)= \left(X_1,X_2,\ldots, X_p\right).\]
Let $Y=\left(Y_1,\ldots,Y_n\right)^T$ and
$\e=\left(\e_1,\e_2,\ldots, \e_n\right)^T\in R^n$. The Lasso
estimate \citep{Tibshirani1996} is then defined as the solution to a penalized least
squares problem (with regularization parameter $\lambda$):
\begin{equation}
\hbet =\arg\min_{\bet} \frac{1}{2n}\|\Y-\X\bet\|_2^2+\lam
\|\bet\|_{1}\label{lasso},
\end{equation}
where for some vector $x \in R^k$, $\|x\|_r = (\sum_{i=1}^k
|x_i|^r)^{1/r}$.

In previous research on the Lasso, the above model has been
assumed where the noise terms are i.i.d.\ and independent of the
predictors (hence homoskedastic). \textit{We call this the standard model.}

\cite{candes2007dss}  suggested that compressed sensing, a sparse
method similar to the Lasso could reduce the number of measurements
needed by medical imaging technology like Magnetic Resonance Imaging (MRI).  This methodology was later
applied to MRI by \cite{lustig2008compressed}. The standard model was useful to their analyses. However, the standard model is not appropriate for other imaging methods such as PET and SPECT \citep{fessler2000statistical}.

PET provides an indirect measure for the metabolic activity of a specific tissue.
To take an image, a biochemical metabolite must be identified that is attractive to the
tissue under investigation.  This biochemical
metabolite is labeled with a positron emitting radioactive material and it is
then injected into the subject.  This substance circulates through the subject, emitting positrons.
When the tissue gathers the metabolite,  the radioactive material concentrates around the tissue.

The positron emissions can be modeled by
a Poisson point process in three dimensions with an intensity rate proportional to the
varying concentrations of the biochemical
metabolite. Therefore, an estimate of the intensity rate is an
estimate of the level of biochemcial metabolite. However, the positron emissions are
not directly observed.   After each positron is emitted it very quickly annihilates a nearby electron, sending two X-ray photons in nearly opposite directions (at the speed of light)
\cite{vardi1985statistical}. These X-rays are observed by several
sensors in a ring surrounding the subject.

A physical model of this system informs the estimation of the intensity level of the
Poisson process from the observed data.
It can be expressed as a Poisson model where the sample size
$n$ represents the number of sensors; $\Y$ is a vector of observed values;
$\truebetaj$ represents the
Poisson intensity rate for a small cubic volume (a voxel) inside the
subject; the design matrix $\X$ specifies the physics of the
tomography and emissions process; and $\p$ is the number of
voxels wanted, the more voxels, the finer the resolution of the
final image.

Because the positron emissions are modeled by a Poisson point process, the variance of each observed value $\Y_i$ is equal to the
expected value $E(\Y_i)$. Motivated by the Poissonian model, this paper studies the
Lasso under the following sparse Poisson-like model:

\begin{equation}
    \label{PL}
    \begin{array}{rll}
    \Y & = & \X\truebeta+\e, \\
    E(\e\mid\X)&=&0, \\
    Cov(\e\mid\X)&=&\const\times diag(|\X\truebeta|),\\
    \e & \ind&\Xb~|~\Xa ,
\end{array}
\end{equation}
where  $\const>0$ and the sparsity index set is defined as
\[ S= \{ 1\leq j \leq p: \beta_j \neq 0\}, \mbox{ with the cardinality $q = \# S$ such that $0< q< p$.}\]
In the definition of the Poisson-like model, $\e$ conditioned on
$\X$ consists of independent Gaussian variables; $Cov(\e\mid\X)$,
the variance-covariance matrix of $\epsilon$ conditioned on $\X$, is
$\const\times diag(|\X\truebeta|)$, an $n\times n$ diagonal matrix with the
vector $\const\times|\X\truebeta|$ down the diagonal; and $\Xa$ and $\Xb$ denote
two matrices consisting of the relevant column vectors ( with nonzero
coefficients) and irrelevant column vectors ( with zero coefficients)
respectively.  This is a heteroscedastic model.

Since the  Lasso provides a computationally feasible way to select a
model \citep{Osborne2000, efron2004least, Rosset2004, Zhao2007}, it
can be applied in the non-standard settings to give sparse
solutions.   In this paper we show that the Lasso is robust to the heteroscedastic
noise of the sparse Poisson-like model.  Under the Poisson-like model, for general scalings of $\p,\q, n,$ and
$\truebeta$, this paper investigates when the Lasso is sign
consistent and when it is not with theoretical and simulation
studies.
Our results are comparable to the results for the standard model;
when a measure of the signal to noise ratio is large, the Lasso is sign consistent.
This is the first study of the sign  consistency of the standard Lasso in a
heteroscedastic setting.

\subsection{Overview of Previous Work}
The Lasso \citep{Tibshirani1996} has been a popular technique to
simultaneously select a model and provide regularized estimated
coefficients.  There is a substantial literature on the use of  the
Lasso for sparsity recovery and subset selection under the standard
homoscedastic linear model. This subsection gives only a very brief overview.

In noiseless setting (when $\e=0$), with contributions from a broad
range of researchers
\citep{ChenDS1998,DonohoH2001,EladB2002,FeuerN2003,CandesRT2004,
Tropp2004}, there is now much understanding of sufficient conditions
on deterministic predictors $\{X_i, i=1,\ldots, n\}$ and sparsity
index $S=\{j:\truebetaj \neq 0\}$ for which the true $\truebeta$ can
be recovered exactly. Results by \cite{Donoho2004}, as well as
\cite{CandesT2005} provide high probability results for random
ensembles of $\X$.
\begin{comment}
More specifically, as independently established
by both sets of authors using different methods, if entries of $\X$
are i.i.d.\  from standard normal distribution $N(0, 1)$, with the
number of predictors  $p$ scaling linearly in terms of the number of
observations (i.e., $p = \gamma n$, for some $\gamma > 1$), there
exists a constant $\alpha > 0$ such that all sparsity patterns with
$q \leq \alpha p$ can be recovered with high probability.
\end{comment}

There is also a substantial body of work focusing on the noisy
setting (where $\e$ is random noise). \cite{KnightF2000} analyze the
asymptotic behavior of the optimal solution for fixed dimension
($p$); not only for $L_1$ regularization, but for $L_r$
regularization with $r \in (0, 2]$. Both \cite{Tropp2006} and
\cite{DonohoET2006} provide sufficient conditions for the support of
the optimal solution to the Lasso problem (\ref{lasso}) to be
contained within the support of $\truebeta$. Recent work on the use
of the Lasso for model selection by \cite{MeinshausenB2006}, focuses
on Gaussian graphical models. \cite{ZhaoY2006} considers linear
regression and more general noise distributions. For the case of
Gaussian noise and Gaussian predictors, both papers established that
under particular mutual incoherence conditions and the appropriate
choice of the regularization parameter $\lam$, the Lasso can recover
the sparsity pattern with probability converging to one for
particular regimes of $ n, p$ and $q$. \cite{ZhaoY2006} used a particular
mutual incoherence condition, the  Irrepresentable Condition, which they
show is almost necessary when $p$ is fixed. The Irrepresentable
Condition was found in \cite{fuchs2005recovery} and
\cite{zou2006adaptive} as well. For i.i.d.\ Gaussian or sub-Gaussian noise,
\cite{Wainwright2009} established a sharp relation between the
problem dimension $p$, the number $q$ of nonzero elements in
$\truebeta$, and the number of observations $n$ that are required
for sign consistency.

\subsection{The Contributions in this Paper}

Before giving the contributions of this paper, some definitions are needed. Define
\[\textrm{sign}(x) =  \left\{\begin{array}{rl}1 & \textrm{ if } x>0 \\0 & \textrm{ if } x=0 \\-1 & \textrm{ if } x<0.\end{array}\right.\]
Define $=_s$ such that $\hbet=_s\truebeta$ if and only if $\mbox{sign}(\hbet) =
\mbox{sign}(\truebeta)$ elementwise.

\begin{definition}
The Lasso is \textbf{sign consistent} if there exists a sequence
$\lambda_n$ such that,
\[P\left(\hat\beta(\lambda_n) =_s \truebeta\right) \rightarrow 1,  \mbox{ as } n \rightarrow \infty.\]
\end{definition}

This paper studies the sign consistency of the Lasso applied to data from the
sparse Poisson-like model, giving non-asymptotic results for both the
deterministic design and the Gaussian random design. The
non-asymptotic results give the probability that
$\hbet=_s\truebeta$, for any $\lam, p, q,$ and $n$. The sign consistency results follow from the non-asymptotic results.   This paper  also gives necessary conditions for the Lasso
to  be sign consistent under the sparse Poisson-like model. It is shown that
the Irrepresentable Condition is necessary for the Lasso's sign
consistency under this model.   This condition is also
necessary under the standard model \citep{ZhaoY2006,
zou2006adaptive,Wainwright2009}. The sufficient conditions for both the
deterministic design and random Gaussian design require that
the variance of the noise is not too large and that  the smallest nonzero element of $|\truebeta|$ is not too small.  Define the smallest nonzero element of $\truebeta$ as
$$\minbeta= \min_{j\in S}|\truebetaj|.$$
For deterministic design, assume that
\begin{equation*}
\Lambda_{\min}\left(\frac{1}{n}\Xa^T\Xa\right)\geq C_{\min}>0,
\label{min.eigen}
\end{equation*}
where $\Lambda_{\min}(\cdot)$ denotes the minimal eigenvalue of a
matrix and $C_{\min}$ is some positive constant; for random Gaussian
design, assume that
\begin{equation*}
\Lambda_{\min}(\SXaXa)\geq \tilde C_{\min} >0  \ \mbox{ and } \
\Lambda_{\max}(\Sigma)\leq \tilde C_{\max} <\infty ,\label{eign}
\end{equation*}
where $\SXaXa \in R^{q \times q}$ is the variance-covariance matrix
of the true predictors, $\Sigma \in R^{p \times p}$ is the
variance-covariance matrix of all predictors,
$\Lambda_{\max}(\cdot)$ denotes the maximal eigenvalue of a matrix,
and $\tilde C_{\min}$ and $\tilde C_{\max}$ are some positive
constants.  An essential quantity for determining the probability of sign recovery is the signal to noise ratio
\begin{equation} \label{SNR}
SNR =   \frac{n [\minbeta]^2}{\const \|\truebeta\|_2}.
\end{equation}
The numerator corresponds to the signal strength for sign recovery. The most difficult sign to estimate in $\truebeta$ is the element that corresponds to $\minbeta$.  When the smallest element is larger, estimating the signs is easier, and the signal is more powerful. The noise term in the denominator contains $\|\truebeta\|_2$ because in the heteroscedastic model considered in this paper, $\|\truebeta\|_2$ is fundamental in the scaling of the noise.  Because this paper is addressing the sign consistency of the Lasso, the definition of the signal in $SNR$ does not correspond to the typical definition.

When $SNR$ is large, the Lasso is sign consistent. Specifically, the sufficient condition for deterministic design requires that
$$SNR  = \Omega\left( \q\log(\p+1)\  \max_i \|x_i(S)\|_2 \right),$$
where $a_n = \Omega(b_n)$ means that $a_n$ grows faster than $b_n$, that is $a_n/b_n\rightarrow \infty$.
The sufficient conditions for random Gaussian design requires that
$$ SNR \ge 8 \log n \sqrt{2\max(4q, \log n)} / \tilde C_{\min} ,$$
and
$$SNR = \Omega\left(  \q \log (\p-\q+1) \right) .$$
The previous three inequalities all require that the signal to noise ratio is large.
This essential result for the Poisson-like model is identical to the corresponding result for the standard model|both require that the variance of the noise is small compared to the size of the signal.  The simulations in Section \ref{Simulations} support this essential result.

Several of the mathematical techniques in this paper come from \cite{Wainwright2009}. However, our proofs are more involved because  the errors in the Poisson-like model are heteroscedastic and dependent on the design matrix. The results in this paper require bounding
$$\frac{\epsilon^T\epsilon}{n^2} \ \ \mbox{  and  }  \ \ {\left\|\Xa^T\epsilon\right\|_{\infty}}.$$
When the errors are Gaussian and homoscedastic, the first quantity is distributed as $\chi^2$.  When the errors are heteroscedastic, the distribution becomes more complicated. This paper calculates the second moment of $\epsilon_i^2$ and bounds  $\epsilon^T\epsilon/n^2$ via the Chebyshev inequality. Similarly for the second quantity, when the errors are dependent of the design matrix, the variance of  $\Xa^T\epsilon$ is more difficult to bound.
In this paper, the variance of $\Xa^T\epsilon$ is bounded using
\[P\left[\max_{i=1,\ldots n} \|x_i(S)\|_2^2 \geq 8\Tilde C_{\max} \max{(4q,\log n)}\right]\leq \frac{1}{n}, \]
where $x_i(S)$ is the $i$th row of $\Xa$. This large deviation result regarding the $\chi^2$ distribution is given in Appendix \ref{LargeDev}.

The remainder of the paper is organized as follows.   Section
\ref{Ddesign} analyzes the Lasso estimator when the design matrix is deterministic.  Section \ref{Rdesign} considers the case where the rows of
$\X$ are i.i.d.\ Gaussian vectors. Both sections give (1) sufficient conditions for the Lasso to be
sign consistent and (2)
necessary conditions for the Lasso's sign consistency. In Section \ref{Simulations}, simulations demonstrate the fundamental role of $SNR$ and show that the Lasso performs similarly on both homoscedastic and Poisson-like data.  Section \ref{conclusion} gives some concluding thoughts. Proofs are presented in Appendix.

\section{Deterministic Design} \label{Ddesign}

This section examines when the Lasso is sign consistent and when it is not sign consistent under the sparse Poisson-like model for a nonrandom design matrix $\X$.  First, some notation,
$$ x_i(S)  = e_i^T \Xa, $$
where $e_i$ is the unit vector with $i$th element one and the rest
zero. Because $S=\{j: \truebetaj\neq 0\}$ is the sparsity index set,
$x_i(S)$ is a row vector of dimension $q$. Define
$$\betA = (\truebetaj)_{j \in S} \ \textrm{ and } \ \overrightarrow{b} = \textrm{sign} (\betA).$$
Suppose the Irrepresentable Condition holds.
That is, for some constant $\eta \in (0,1]$,
\begin{equation}
\left\|\Xb^T\Xa\left(\Xa^T\Xa\right)^{-1}\overrightarrow{b}\right\|_{\infty}\leq
1-\eta. \label{IC}
\end{equation}
The $\ell_\infty$ norm of a vector, $\| \cdot \|_\infty$, is defined
as the vector's largest element in absolute value. In addition,
assume that
\begin{equation}
\Lambda_{\min}\left(\frac{1}{n}\Xa^T\Xa\right)\geq C_{\min}>0,
\label{min.eigen}
\end{equation}
where $\Lambda_{\min}$ denotes the minimal eigenvalue and $C_{\min}$
is some positive constant.  Condition \eqref{min.eigen} guarantees
that matrix $\Xa^T\Xa$ is invertible. These conditions are also
needed in \cite{Wainwright2009} for sign consistency of the Lasso
under the standard model. Define
$$\Psi(\X, \truebeta, \lam) =\lam\left[ \eta \ (C_{\min})^{-1/2}+ \left\|\left(\frac{1}{n}\Xa^T\Xa \right)^{-1} \overrightarrow{b} \right\|_\infty \right],$$
with which:

\begin{theorem}{\label{theorem:deterministic}}
Suppose that data $(\X,\Y)$ follows the sparse Poisson-like model described by
Equations \eqref{PL} and each column of $\X$ is normalized to
$l_2$-norm $\sqrt{n}$. Assume that (\ref{IC}) and (\ref{min.eigen})
hold. If $\lam$ satisfies
$$\minbeta>  \Psi(\X, \truebeta, \lam),$$
then with probability greater than
$$1-2\exp\left\{-\frac{n\lam^2\eta^2}{2\const \|\truebeta\|_2 \max_{1\leq i\leq n}\|x_i(S)\|_2}+\log(p)\right\},$$
the Lasso has a unique solution $\hbet$ with $\hbet =_s \truebeta$.

\end{theorem}

Theorem \ref{theorem:deterministic} can be thought as a straightforward result from Theorem 1 in \cite{Wainwright2009}.
In \cite{Wainwright2009}, sign consistency of the Lasso estimate is given
for a standard model with sub-Gaussian noise with parameter $\sigma^2$. In the Poisson-like model, since $var(\epsilon_i|x_i) = \sigma^2 |x_i^T \truebeta| \leq \sigma^2 \max_i \|x_i(S)\|_2\|\truebeta\|_2$,
the noise can be thought of as sub-Gaussian variables with parameter $\const \max_i \|x_i(S)\|_2\|\truebeta\|_2$. To make this paper self-contained, a proof of Theorem \ref{theorem:deterministic} is given in Appendix \ref{detApp}.

Theorem \ref{theorem:deterministic} gives a non-asymptotic result on
the Lasso's sparsity pattern recovery property.  The next corollary specifies a sequence of $\lam$'s that can asymptotically recover the true sparsity pattern.
The essential requirements are that
\[\mbox{(1) }  \frac{n\lambda^2}{ \max_i \|x_i(S)\|_2\|\truebeta\|_2\log(p+1)} \rightarrow \infty \ \mbox{  and (2)  }  \  \minbeta>  \Psi(\X, \truebeta, \lam).\]
Define,
\[ \Gamma(\X, \truebeta, \const) = \frac{\eta^2 \ SNR }{8\max_i\|x_i(S)\|_2 (\eta \  C_{\min}^{-1/2}+\sqrt{q} \ C_{\min}^{-1})^2\log(\p+1)}.\]
\begin{corollary}\label{corollary:general:fixed}
As in Theorem \ref{theorem:deterministic}, suppose that data
$(\X,\Y)$ follows the sparse Poisson-like model described by Equations
\eqref{PL} and each column of $\X$ is normalized to $l_2$-norm
$\sqrt{n}$. Assume that (\ref{IC}) and (\ref{min.eigen}) hold.
Take $\lam$ such that
\begin{equation}\label{lambdaDefinition}
\lam=\frac{\minbeta }{ 2\left(\eta \
C_{\min}^{-1/2}+\sqrt{q} \ C_{\min}^{-1}\right)},
\end{equation}
then $\hbet =_s \truebeta$ with probability greater
than
$$1-2\exp\left\{ -\left(\Gamma(\X, \truebeta, \const,\alpha) -1 \right)\log(\p +1)\right\}.$$
If $\Gamma(\X, \truebeta, \const) \rightarrow \infty$, then $P[\hbet =_s \truebeta]$ converges to one.
\end{corollary}

A proof of this corollary can be found in Appendix \ref{proof:general:fixed}.

This corollary gives a class of heteroscedastic models  for which
the Lasso gives a sign consistent  estimate of $\truebeta$.
This class requires that $\Gamma(\X,
\truebeta, \const)\rightarrow \infty$ which means that
\begin{equation}\label{nsr:deterministic}
SNR = \frac{n[\minbeta]^2}{ \const\|\truebeta\|_2} = \Omega\left({ q \log(\p+1) \max_i\|x_i(S)\|_2}\right),
\end{equation}
where $a_n = \Omega(b_n)$ means that $a_n$ grows faster than $b_n$, that is, $a_n/b_n\rightarrow \infty.$ In other words, this condition requires that $SNR$ grows fast enough.

For a moment, suppose that the errors are homoscedastic and $var(\epsilon_i) = \const$.  The exact same theorem could be proven for this homoscedastic model by  replacing $\const\|\truebeta\|_2$ with $\const$ and removing $\max_i\|x_i\|_2$.  This shows that when a measure of the signal to noise ratio is large, the Lasso can select the true model in both the case of homoscedastic noise and Poisson-like noise.

The next corollary addresses the classical setting, where $p,q,$ and
$\truebeta$ are all fixed and $n$ goes to infinity. While this is a
straightforward result from Corollary \ref{corollary:general:fixed},
it removes some of the complexities and leads to good intuition.
Since $\minbeta$ and $\|\truebeta\|_2$ do
not change with $n$, $\Gamma(\X, \truebeta, \const, \alpha)
 \rightarrow \infty$ in Corollary \ref{corollary:general:fixed} when
$\frac{1}{n} \max_{1\leq i\leq n}\|x_i(S)\|_2\rightarrow 0$.
Then:
\begin{corollary} \label{corollary:fixed beta}
As in Theorem \ref{theorem:deterministic}, suppose that data
$(\X,\Y)$ follows the sparse Poisson-like model described by Equations
\eqref{PL} and each column of $\X$ is normalized to $l_2$-norm
$\sqrt{n}$. Assume that (\ref{IC}) and (\ref{min.eigen}) hold.  In
the classical case when $\p, \q$ and $\truebeta$ are fixed, if
\begin{equation}
\frac{1}{n} \hspace{.03 in} \max_{1\leq i\leq
n}\|x_i(S)\|_2\rightarrow 0, \label{normtozero}
\end{equation}
then by choosing $\lam$ as in equation \eqref{lambdaDefinition},
$$P\left[\hbet=_s\truebeta\right]
\rightarrow 1$$ as $n\rightarrow \infty$.
\end{corollary}

Condition (\ref{normtozero}) is not strong and it is easy to be
satisfied. Suppose
\[
0<\Lambda_{\max}\left(\frac{1}{n}\Xa^T\Xa\right)\leq C_{\max} < \infty,
\]where $\Lambda_{\max}(\cdot)$ is the maximum eigenvalue  of a matrix and $C_{\max}$ is a positive constant,
then
\[
\left
\|\frac{1}{\sqrt{n}}x_i(S)\right\|_2^2 = \left\|\frac{1}{\sqrt{n}}e_i^T\Xa\right\|_2^2\leq
\Lambda_{\max}\left(\frac{1}{n}\Xa^T\Xa\right)\leq C_{\max}.
\]
Consequently, $$\frac{1}{n}\max_{1\leq i\leq
n}\left\|x_i(S)\right\|_2=\frac{1}{\sqrt{n}}\max_{1\leq i\leq
n}\left\|\frac{1}{\sqrt{n}}x_i(S)\right\|_2\leq
\frac{1}{\sqrt{n}}C_{\max}^{1/2}\rightarrow 0.$$  Corollary
\ref{corollary:fixed beta} states that in the classical settings,
the Lasso can consistently select the true model under the
sparse Poisson-like model.

So far the results have given sufficient conditions for sign
consistency of the Lasso. To further understand how the sign consistency of the Lasso might be
sensitive to the heteroscedastic model, the next theorem gives
necessary conditions on  the ratio of  $\truebetaj$ to the noise level.

\begin{theorem}[Necessary Conditions]\label{theorem:NC}
Suppose that data $(\X,\Y)$ follows the sparse Poisson-like model described by
Equations \eqref{PL} and each column of $\X$ is normalized to
$l_2$-norm $\sqrt{n}$. Assume that (\ref{min.eigen}) holds.
\begin{enumerate}[(a)]
\item Consider $\frac{1}{n}\Xa^T\Xa=I_{q\times q}$. For any $j$, define
\begin{equation}\label{cn}
c_{n,j}^2 =\frac{n^2\truebetaj^2}{\const e_j^T\left[\Xa^T
diag(|X\truebeta|)\Xa\right] e_j}.
\end{equation}
Define $c_n = \min_j c_{n,j}$. Then, for sign consistency, it is
necessary that $c_{n} \rightarrow \infty$. Specifically,
$$P\left[\hbet =_s \truebeta \right]\leq 1-\frac{\exp\left\{-c_{n}^2/2\right\}}{\sqrt{2\pi}(1+c_{n})}.$$
\item If the Irrepresentable Condition \eqref{IC} does not hold, specifically, \begin{equation}
\left\|\Xb^T\Xa\left(\Xa^T\Xa\right)^{-1}\overrightarrow{b}\right\|_{\infty} \ge
1, \label{NIC}
\end{equation}
then, the Lasso estimate is not sign consistent: $P\left[\hbet =_s
\truebeta\right]\leq \frac{1}{2}.$
\end{enumerate}
\end{theorem}

A proof of Theorem \ref{theorem:NC} can be found in Appendix \ref{proof:NC}.

Statement (a) would hold for the homoscedastic model by removing $diag(|\X \truebeta|)$ from the denominator in Equation \eqref{cn}.  %But under the Poisson-like model, $SNR$ is
%sensitive to  the scale of $\truebeta$; as $\truebeta$ becomes large, so does the noise.  As statement (a) underscores both the signal and the noise are sensitvie to
%$\truebeta$ under the Poisson-like model.  The effects of $\minbeta$ and $\| \truebeta\|_2$ for a fixed sample size will be investigated in Section \ref{Simulations}.
Equation \eqref{cn} can be viewed as a comparison of the signal strength ($\truebetaj^2$) to the noise level ($var(X_j^T\epsilon))$. Theorem \ref{theorem:NC} shows that the signal strength needs to be large relative to the noise level.

Statement (b) says that the Irrepresentable
Condition \eqref{IC} is necessary for the Lasso's sign consistency.
This necessary condition can also be found in both \cite{ZhaoY2006}
and \cite{Wainwright2009}. \cite{ZhaoY2006} points out that the
Irrepresentable Condition is almost necessary and sufficient for the
Lasso to be sign consistent under the standard homosedastic model when $\p$ and
$\q$ are fixed. \cite{Wainwright2009} says that it is necessary for
the Lasso's sign consistency under the standard model for any $\p$
and $\q$.

The results in this section have addressed the case when $\X$ is fixed. The essential results say that when $SNR$ is large, the Lasso performs well at estimating sign$(\truebeta)$. In the next section $\X$ is random. The randomness of $\X$ allows us to study how the statistical dependence between $\X$ and the noise terms affects the sign consistency of the Lasso.  When $SNR$ is large, the Lasso is robust to this violation of independence.

\section{Gaussian Random Design}
\label{Rdesign} Consider the Gaussian random design where rows
of $\X$ are i.i.d.\  from a $p$-dimensional multivariate Gaussian
distribution with mean $0$ and variance-covariance matrix $\Sigma$.
Assume the diagonal entries of $\Sigma$ are all equal to one. Define the variance-covariance
matrix of the relevant predictors to be $\SXaXa$ and the covariance
between the irrelevant predictors and the relevant predictors to be
$\SXbXa$. Specifically,
\begin{eqnarray*}
\SXaXa &=& E\left(\frac{1}{n}\Xa^T \Xa\right) \ \textrm{ and}\\
\SXbXa &=&  E\left(\frac{1}{n}\Xb^T \Xa\right).
\end{eqnarray*}
Let $\Lambda_{\min}(\cdot)$ denote the minimum eigenvalue of a
matrix and $\Lambda_{\max}(\cdot)$ denote the maximum eigenvalue of
a matrix. To get the main results that allow $p$ to grow with $n$,
the following regularity conditions are needed on the $p \times p$
matrix $\Sigma$. First, for some positive constants $\tilde C_{\min}$ and
$\tilde C_{\max}$ that do not depend on $n$,
\begin{equation}
\Lambda_{\min}(\SXaXa)\geq \tilde C_{\min} > 0\ \mbox{ and } \
\Lambda_{\max}(\Sigma)\leq \tilde C_{\max} < \infty,\label{eign}
\end{equation}
and second, the Irrepresentable Condition,
\begin{equation}
\|\SXbXa(\SXaXa)^{-1}\overrightarrow{b}\|_{\infty}\leq 1-\eta \label{IC2},
\end{equation}
for some constant $\eta \in (0,1]$.  Assumptions \eqref{eign} and \eqref{IC2} are standard assumptions in the previous research under standard models.
Define,
\begin{eqnarray*}
V^*(n, \truebeta, \lam, \const) &=&  \frac{2\lam^2 q}{n\tilde C_{\min}}+\frac{3\const\sqrt{\tilde C_{\max}}\|\truebeta\|_2}{n},\\
A(n,\truebeta,\const)&=& \sqrt{\frac{4\const\|\truebeta\|_2\log n\sqrt{2\max(16q, 4\log n)}}{n\tilde C_{\min}}}\textrm{ and} \\
\tilde \Psi(n, \truebeta, \lam, \const)&= &
A(n,\truebeta,\const)+\frac{2\lam\sqrt{q}}{\tilde C_{\min}}.
\end{eqnarray*}
These quantities defined above are used in the following theorem.
\begin{theorem}{\label{theorem:random:ensemble}}
Consider the sparse Poisson-like model described by \eqref{PL}, under
Gaussian random design. Suppose that the variance-covariance matrix
$\Sigma$ satisfies condition (\ref{eign}) and condition \eqref{IC2}
with unit diagonal. Further, suppose that $\q/n\rightarrow 0$. Then
for any $\lam$ such that
$$\minbeta> \tilde \Psi(n, \truebeta, \lam, \const) ,$$
$\hbet =_s \truebeta$ with probability greater than
$$1-2\exp\left\{-\frac{\lam^2\eta^2}{2V^*(n, \truebeta, \lam, \const)\tilde C_{\max}}+\log(\p-\q)\right\}-(2q+3)\exp\{-0.03n\}-\frac{1+3q}{n}.$$

\end{theorem}

A proof of Theorem \ref{theorem:random:ensemble} can be found in
Appendix \ref{random:ensemble:proof}.

Theorem \ref{theorem:random:ensemble} gives a non-asymptotic result on
the Lasso's sparsity pattern recovery property, from which the next corollary can be derived. It
specifies a sequence of $\lam$'s that asymptotically recovers the true sparsity pattern
on a well behaved class of models.  This
class of models restricts the relationship between the data $(\X)$,
the coefficients ($\truebeta$), and the distribution of the noise
($\e$). Basically speaking, $\lambda$ should be chosen such that
\[\mbox{(1) }  \frac{\lam^2\eta^2}{2V^*(n, \truebeta, \lam, \const)\tilde C_{\max}}- \log(\p-\q) \rightarrow \infty \ \mbox{  and (2)  }  \  \minbeta>  \tilde\Psi(n, \truebeta, \lam,\const).\]
Define
\begin{eqnarray*}
&& \tilde{\Gamma}(n,\truebeta,\const)= \\
&&n\eta^2 \left[ 4q\log(p-q+1)\tilde C_{\max}/\tilde{C}_{\min}+\frac{96\const q \|\truebeta\|_2 \log (\p-\q+1) \sqrt{\tilde C_{\max}^3}}{[\minbeta - A(n,\truebeta,\const)]^2\tilde{ C}_{\min}^2}\right]^{-1} .
\end{eqnarray*}
\begin{corollary}\label{corollary:general:random}
As in Theorem \ref{theorem:random:ensemble}, consider the
sparse Poisson-like model described by \eqref{PL}, under Gaussian random
design. Suppose the variance-covariance matrix $\Sigma$ satisfies
condition (\ref{eign}) and condition \eqref{IC2} with unit diagnal.
Further, suppose that $\minbeta > A(n,\truebeta,\const)$  and $\q/n\rightarrow 0$.
Take $\lam$ such that
$$\lam=\frac{[\minbeta - A(n,\truebeta,\const)]\tilde{C}_{\min}}{ 4\sqrt{q}},$$
then  $\hbet =_s \truebeta$ with probability greater than
\begin{equation*}
1-2\exp\left\{-{\log(p-q+1)[\tilde\Gamma(n,\truebeta,\const) - 1]}\right\} -(2q+3)\exp\{-0.03 n\}-\frac{1+3q}{n}.
\end{equation*}
If
\begin{equation}
n/[q\log(p-q+1)]\rightarrow \infty \  \mbox{   and   } \ \frac{n[\minbeta - A(n,\truebeta,\const)]^2}{\const  \|\truebeta\|_2\ \q \log(\p-\q+1)}\rightarrow
\infty,\label{qnp}
\end{equation} then $P[\hbet =_s \truebeta]$
converges to one.
\end{corollary}

A proof of Corollary \ref{corollary:general:random} can be found in
Appendix \ref{proof:general:random}.

This corollary gives a class of heteroscedastic models  for which
 the Lasso gives a sign consistent estimate of $\truebeta$, when the predictors are from a Gaussian
random ensemble.

The sufficient
conditions require that $\minbeta\geq A(n,\truebeta, \const)$, which
is equivalent to
\begin{equation} \label{nsr:random1}
SNR \geq 8 \tilde C_{\min}^{-1}\log n \sqrt{2\max(4q, \log n)}.
\end{equation}
The sufficient conditions also require the conditions in \eqref{qnp}, which
imply that
\begin{equation} \label{nsr:random2}
SNR = \Omega\left({q\log(p-q+1)}\right).
\end{equation}
These conditions show that when $SNR$ is large, the Lasso can identify the sign of the true predictors.  This result is similar to the result for the fixed design case in the previous section and it is similar to results on the standard model.

The next theorem gives necessary
conditions for the Lasso to be sign consistent. It says that the Irrepresentable Condition is necessary for the sign consistency of the Lasso under the sparse Poisson-like model. This condition is also necessary under the homoscedastic model.

\begin{theorem}[Necessary Conditions]\label{theorem:NC:random}
Consider the sparse Poisson-like model described by \eqref{PL}, under
Gaussian random design. Suppose the variance-covariance matrix
$\Sigma$ satisfies condition (\ref{eign}).
If the Irrepresentable Condition \eqref{IC2} does not hold, specifically, \begin{equation}
\|\SXbXa(\SXaXa)^{-1}\overrightarrow{b}\|_{\infty} \ge 1, \label{NIC}
\end{equation}
then, the Lasso estimate is not sign consistent: $P\left[\hbet =_s
\truebeta \right]\leq \frac{1}{2};$
\end{theorem}

A proof of Corollary \ref{theorem:NC:random} can be found in
Appendix \ref{proof:NC:random}.

This section identified sufficient conditions for the Lasso to be sign consistent when the design matrix is random and Gaussian.
The sufficient conditions are similar for both fixed and random design matrices. They are also similar for both homoscedastic noise and Poisson-like noise.
In all cases, the conditions require that a measure of the signal to noise ratio is large, see Equations \eqref{nsr:deterministic} and \eqref{nsr:random2} and the inequality in \eqref{nsr:random1}  . %and \eqref{nsr:random2}.
In the next section, simulations are used to directly compare the performance of the Lasso between the Poisson-like model and the standard homoscedastic model.

\section{Simulation Studies} \label{Simulations}
There are two examples in this section. The first example investigates a peculiarity of the $SNR$ defined in \eqref{SNR};  functions of $\truebeta$ appear in both the signal and in the noise.  The second example compares the model selection performance of the Lasso under the standard model to the model selection performance of the Lasso under the sparse Poisson-like model.
In the first example, all data is generated from the sparse Poisson-like model.  In the second example, the performance of the Lasso is compared between the two models of the noise.     The parameterizations of the standard homoscedastic models differ in only one respect|the noise terms are homoscedastic.  To ensure a fair comparison, the variance of the noise terms in the standard model is set equal to the average variance of the noise terms in the corresponding Poisson-like model.

All simulations were done in R with the LARS package \citep{efron2004least}.

\Example{1} This example studies how the
Lasso is sensitive to the ratio of signal to noise defined earlier. Recall,
\begin{equation}\label{cndef}
SNR = \frac{n[\minbeta]^2}{\const \|\truebeta\|_2}.
\end{equation}
The  theoretical results in the previous sections
suggest that when $SNR$ is large,  the Lasso is
sign consistent.  In $SNR$, $\truebeta$ can affect the ratio in two ways.  As $\|\truebeta\|_2$ grows, so does the variance of the noise term. As $[\minbeta]^2$ shrinks, so does the signal. This first example  investigates both effects, changing the value of $\|\truebeta\|_2$ while keeping $\minbeta$ fixed, and vice-versa.

 Consider an initial model with the parameters such that
$n = 400$, $p = 1000$, $q=20$, $\const = 1$, and each element of the
design matrix $\X$ is drawn independently from $N(0,1)$.
Once $\X$ is drawn, it will be fixed through all of the simulations.
This $\X$ is also used in Example 2.
 $\truebeta$ is designed this way:
\[\truebetaj = \left\{\begin{array}{ll} \beta_{\max} & \textrm{ if } j \le 10 \\ \beta_{\min} & \textrm{ if }  11 \le j \le 20 \\0 & \textrm{ otherwise. } \end{array}\right.\]
Changing $\beta_{\min}$ or $\|\truebeta\|_2$ changes the value
of $SNR$. To investigate the role of these two quantities, there are two simulation designs.  The first simulation design fixes $\minbeta=\beta_{\min}=5$ and changes the value of $\beta_{\max}$.  The second simulation design fixes $\|\beta\|_2$ and changes the value of $\minbeta$. There is one model that is present in both designs. It sets $\beta_{\max} = 40$ and $\beta_{\min} = 5$. In this model that is common to both designs, $\|\beta\|_2=127$ and $SNR=400\times 5^2/127 \approx 78$.

The first simulation design has ten different parameterizations. It sets $\beta_{\min} = 5$ and chooses
$$\beta_{\max}\in
\{100,90,80,70,60,50,40,30,20,10\}.$$
Each of these ten different parameterizations creates a different value of $SNR$. The second simulation design has ten different parameterizations, each fixing $\|\truebeta\|_2$ and altering  $\beta_{\min}$ such that  $SNR$
does not change from the first simulation design (to keep $\|\truebeta\|_2$ fixed,  $\beta_{\max}$ must change accordingly). The values of the parameters
for the two designs are described in the following two tables.

\begin{table}[ht]
\begin{center}
{\caption{The design of the first simulation is described in this table.
It shows the relationship between $\|\truebeta\|_2,
\minbeta$, and $SNR$.  $\minbeta=5$
is fixed. When $\beta_{\max}$ shrinks, $\|\truebeta\|_2 $ also shinks,
increasing $SNR$. The numbers in the table are rounded to the nearest integer. }}

\vspace{.1in}

\begin{tabular}{l|rrrrrrrrrr}
  \hline
$\beta_{\max}$ & 100&90&80&70&60&50&40&30&20&10\\
 $\|\truebeta\|_2$ & 317 &285 & 253 & 222 & 190 & 159 & 127 & 96 & 65 & 35 \\
 $SNR$ & 32&  35 & 39&  45&  53&  63&  78& 104& 153& 283 \\
   \hline
\end{tabular}
\end{center}
\end{table}

\begin{table}[ht]
\begin{center}
{\caption{The design of the second simulation is described in this table.
It shows the relationship between $\minbeta, \beta_{\max}$, and $SNR$. $ \beta_{\min}$ and $\beta_{\max}$ are chosen such that $\|\truebeta\|_2=127$ is fixed and $SNR$ keeps the same values
as in simulation design one. Thus, $\beta_{\min}$ and $\beta_{\max}$ are decided
by $SNR$ and $\|\truebeta\|_2$. $\beta_{\min}=
\sqrt{SNR \|\truebeta\|_2/n}$ and
$\beta_{\max}=\sqrt{\|\truebeta\|_2^2/10-\beta_{\min}^2}$. The numbers in the table are rounded.}}

\vspace{.1in}

\begin{tabular}{l|rrrrrrrrrr}
  \hline
  $\beta_{\min} = \minbeta$ &3.2& 3.3& 3.6& 3.8& 4.1& 4.5& 5.0& 5.8& 7.0 & 9.5\\
$\beta_{\max}$ & 40 &  40 &  40 &  40 &  40 &  40 &  40 &  40 &  40 &  39\\
 $SNR$ & 32&  35 & 39&  45&  53&  63&  78& 104& 153& 283 \\
   \hline
\end{tabular}
\end{center}
\end{table}

For each simulation design, the Monte Carlo estimate for the probability of correctly estimating the signs is plotted against $SNR$ in Figure
\ref{Figure:cnFixed}.  Each point along the solid line in
Figure \ref{Figure:cnFixed} corresponds to simulation design one ($\beta_{min}$ is fixed), and  each point along the dashed line corresponds to simulation design two ($\|\truebeta\|_2$ is fixed).  Success is defined as the existence of a $\lam$ that makes $\hbet =_s \truebeta$. The probability of success for each point is estimated with $500$ trials.

Figure \ref{Figure:cnFixed} shows that as $SNR$ increases, the probability
of success also increases. What is especially remarkable is the similarity between the solid
and dashed lines.  This simulation demonstrates that increasing the elements of $\truebeta$ can both increase and decrease the probability of successfully estimating the signs.  Further, this simulation demonstrates that these effects are well characterized by $SNR$.

\begin{figure}[h!]
\begin{center}
\scalebox{0.6}[0.6]{\includegraphics[]{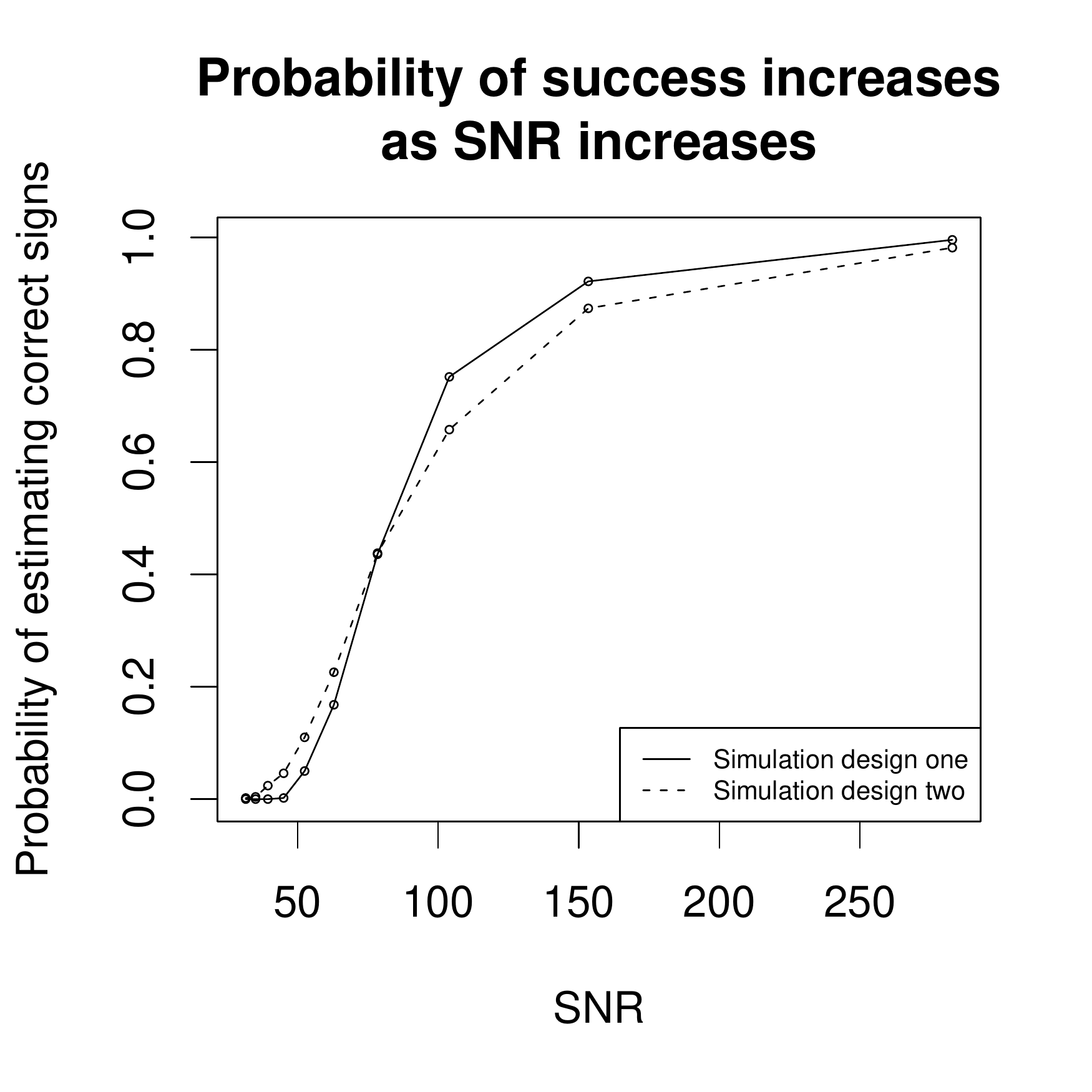}} \vspace{-5 mm}
\caption{
Probability of Success vs. $SNR$ in Example 1. For the solid line,
$\|\truebeta\|_2$ decreases while $\minbeta$ is kept constant. For
the dashed line, $\minbeta$ increases while $\|\truebeta\|_2$ is
kept constant.  The values of $\minbeta$ and $\|\truebeta\|_2$ are
chosen so that $SNR$ as defined in \eqref{cndef} takes the values
specified on the horizontal axis. Each probability is estimated with
500 simulations. } \label{Figure:cnFixed}
\end{center}
\end{figure}

\Example{2 (Comparison to Standard Model)} This example compares the performance of the Lasso applied to homoscedastic data to the performance of the Lasso applied to Poisson-like data.
The design matrix $\X$ is exactly the same (fixed) matrix as in Example 1 and $\truebeta$ follows the constructions specified in Tables (1) and (2).  The only difference between  Example 1 and Example 2 is that the noise terms are homoscedastic in Example 2.
To ensure a fair comparison, the variance of the noise is always set equal to the average variance  of the noise terms in the corresponding Poisson-like model.

There are four lines drawn in Figure \ref{Figure:growBeta}. The solid lines correspond to simulation design one ($\| \truebeta \|_2$ grows while $\minbeta$ is held constant). The dashed lines corresponds to simulation design two ($\minbeta$ shrinks while $\| \truebeta \|_2$ is held constant). The bold lines correspond to the simulations on homoscedastic data. The thinner lines are identical to the lines in Figure \ref{Figure:cnFixed}. They are included to compare the performance of the Lasso on homoscedastic data to the performance of the Lasso on Poisson-like data.

Exactly as in Example 1, success is defined as the existence of a $\lam$ which makes $\hbet
=_s \truebeta$. The probability of success for each point is
estimated with 500 trials.

\begin{figure}[h!]
\begin{center}
\scalebox{.6}{\includegraphics[]{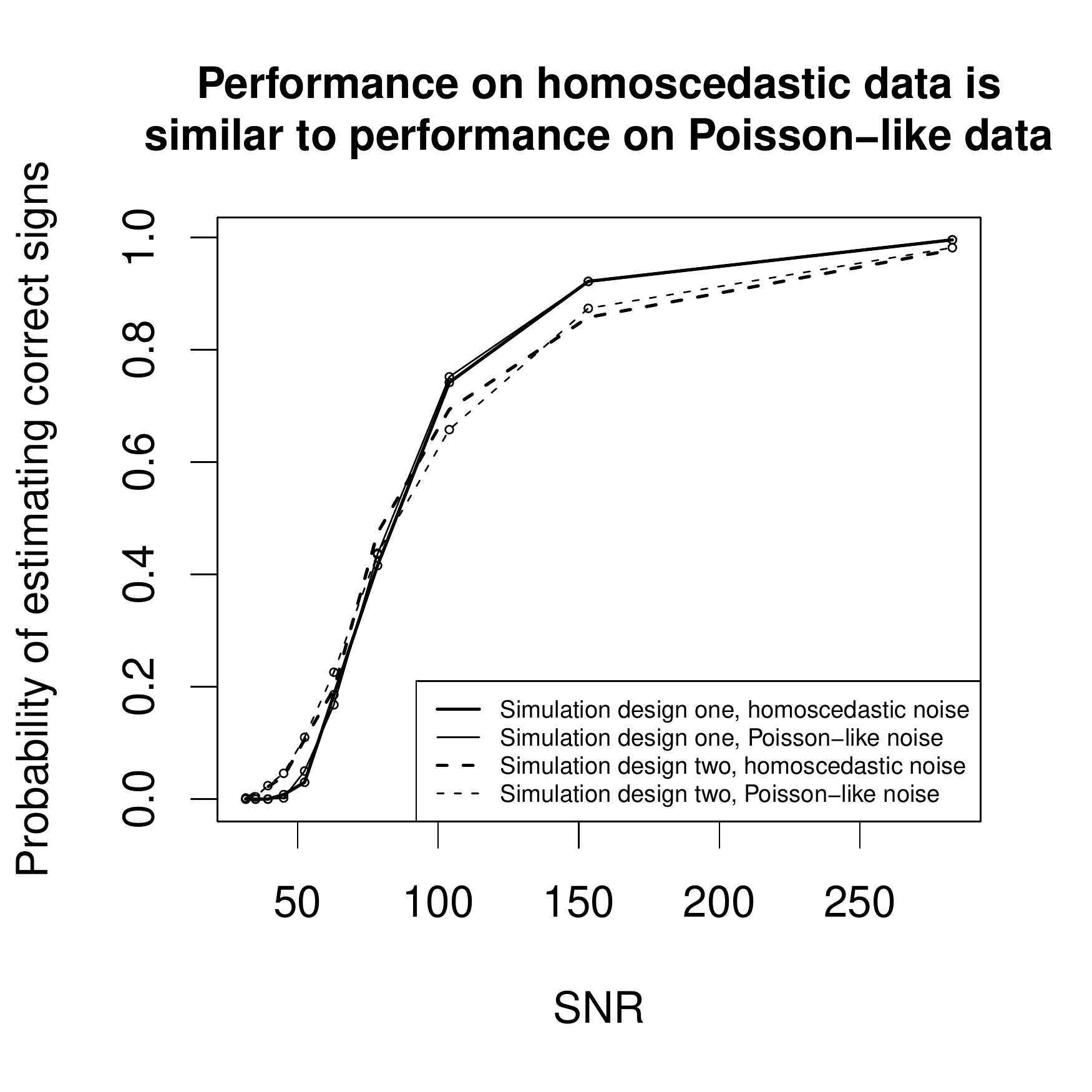}}
 \vspace{-5 mm}
 \caption{
 The bold solid line nearly covers the thinner solid line. This demonstrates how similar the results are for the homoscedastic data and the Poisson-like data. The same statement holds for the dashed lines.
 } \label{Figure:growBeta}
\end{center}
\end{figure}

In the Poisson-like model, the variance of the noise term is dependent on the predicting variables. In the standard model, the noise term is homoscedastic, independent of the predicting variables. Yet, Figure \ref{Figure:growBeta} demonstrates how similar the performance of the Lasso is under  these two different models of the noise. In the figure, the thinner lines are nearly indistinguishable from the bold lines, showing that the Lasso is robust to one type of heteroscedastic noise.

\section{Conclusion}
\label{conclusion}

There is a significant amount of research dedicated to understanding how the Lasso (and similar methods) perform under the standard homoscedastic model. However, in practice, data does not necessarily have the features described by the standard model. This paper aims to understand if the sign consistency of the Lasso is robust to the heteroscedastic errors in the Poisson-like
model. This model is motivated by certain problems in high-dimensional medical imaging \citep{fessler2000statistical}. The key feature of the model is that, for each observation, the variance of the noise term is proportional to the absolute value of the expectation of the observation ($var(\epsilon_i) \propto |E(Y_i)|$). In the Poisson-like model, like all heteroscedastic models, the noise term is \textit{not} independent of the design matrix.

In this paper, we analyzed the sign consistency of the standard Lasso when the data comes from the sparse Poisson-like model, showing that the Lasso is robust to one type of violation to the assumption of homoscedasticity. Theorems \ref{theorem:deterministic} and \ref{theorem:random:ensemble} give non-asymptotic results for the Lasso's sign consistency property under the sparse Poisson-like model for both deterministic design matrices and  random Gaussian ensemble design matrices. Followed by these non-asymptotic results, a suitable $\lam$ was chosen such that
under sufficient conditions the Lasso is sign consistent.   We also studied
how sensitive the Lasso is to the heteroscedastic model by finding
necessary conditions for sign consistency.  The theoretical results for the sparse Poisson-like model are similar to results for the standard model when $SNR$ are matched.  In both models, for the Lasso to be sign consistent, it is essential that a measure of the signal to noise ratio is large.  The simulations demonstrated what our theory predicted, the Lasso performs similarly in terms of  model selection performance on both Poisson-like data and homoscedastic data when the variance of the noise is scaled appropriately. These simulations were across multiple choices of $\truebeta$ and multiple choices of the noise level.  Taken as a whole, these results suggest that the Lasso is robust to Possion-like heteroscedastic noise.

\vspace{11pt} %\noindent {\large\bf Acknowledgments}
\noindent {\large\bf Acknowledgment}
{ This work is inspired by a personal communication between Bin
Yu and Professor Peng Zhang from Capital Normal University in
Beijing. We would like to thank Professor Ming Jiang and Vincent Vu
for their helpful comments and suggestions on this paper. Jinzhu Jia and Bin Yu
are  partially supported by
 NSF grants DMS-0605165 and SES-0835531.
Karl Rohe is partially supported by a NSF VIGRE Graduate Fellowship.
Bin Yu is also partially supported by a grant from MSRA.}
\par

%%%%==========
\vskip 0.2in

\appendix
\noindent{\huge\bf Appendix}
\section{Proofs}
\subsection{Proof of Theorem 1}
\label{detApp}
%\section*{Appendix A.}

To prove the theorem, we need the next Lemma which gives necessary
and sufficient conditions for  the Lasso's sign consistency. They are
important to the asymptotic analysis. \cite{Wainwright2009} gives
this condition which follows from KKT conditions.

\begin{lemma}{\label{KKT}}For linear model $Y=\X\truebeta+\epsilon$,
assume that the matrix $\Xa^T\Xa$ is invertible. Then for any given
$\lam>0$ and any noise term $\e\in R^n$, there exists a Lasso
estimate $\hbet$ which satisfies $\hbet=_s\truebeta$, if and only if
the following two conditions hold
\begin{equation}
\left|\Xb^T\Xa(\Xa^T\Xa)^{-1}\left[\frac{1}{n}\Xa^T\e-\lam
\sign(\betA)\right]-\frac{1}{n}\Xb^T\e\right|\leq \lam, \label{R1}
\end{equation}
\begin{equation}
sign\left(\betA+(\frac{1}{n}\Xa^T\Xa)^{-1}\left[\frac{1}{n}\Xa^T\e-\lam
sign(\betA)\right]\right)=\sign(\betA), \label{R2}
\end{equation}
where the vector inequality and equality are taken elementwise.
Moreover, if \eqref{R1} holds strictly, then
$$\hat{\beta}=(\hbetA,0)$$ is the unique optimal solution to the Lasso problem \eqref{lasso}, where
\begin{equation}
\hbetA=\betA+(\frac{1}{n}\Xa^T\Xa)^{-1}\left[\frac{1}{n}\Xa^T\e-\lam
\sign(\truebeta)\right]. \label{R3}
\end{equation}
\end{lemma}

As in \cite{Wainwright2009}, we state sufficient conditions for
\eqref{R1} and \eqref{R2}. Define
$$\overrightarrow{b}=\sign(\betA),$$ and denote by $e_i$ the vector
with $1$ in the $i$th position and zeroes elsewhere. Define
$$U_i=e_i^T(\frac{1}{n}\Xa^T\Xa)^{-1}\left[\frac{1}{n}\Xa^T\e-\lam\overrightarrow{b}\right],$$
$$V_j=X_j^T\left\{\Xa(\Xa^T\Xa)^{-1}\lam\overrightarrow{b}-\left[\Xa(\Xa^T\Xa)^{-1}\Xa^T-I)\right]\frac{\e}{n}\right\}.$$
By rearranging terms, it is easy to see that (\ref{R1}) holds
strictly if and only if
\begin{equation}
\mathcal{M}(V)=\left\{\max_{j\in S^c} |V_j|< \lam\right\} \label{C1}
\end{equation} holds. If we define $\minbeta=
\min_{j\in S}|\truebetaj|$ (recall that $S=\{j: \truebetaj\neq 0\}$
is the sparsity index), then the event
\begin{equation}
\mathcal{M}(U)=\left\{\max_{i \in S }|U_i|<
\minbeta\right\},\label{C2}
\end{equation}
is sufficient to guarantee that condition (\ref{R2}) holds.
Finally, a proof of Theorem \ref{theorem:deterministic}.

\vspace{.2 in}

\begin{proof}
This proof is divided into two parts. First we analysis the
asymptotic probability of event $\mathcal{M}(V)$, and then we
analysis the event of $\mathcal{M}(U)$.

\textbf{Analysis of $\mathcal{M}(V):$} Note from \eqref{C1} that
$\mathcal{M}(V)$ holds if and only if    $\frac{\max_{j \in S^c}
|V_j|}{\lam}< 1$. Each random variable $V_j$ is Gaussian with mean
$$\mu_j=\lam X_j^T\Xa(\Xa^T\Xa)^{-1}\overrightarrow{b}.$$

Define
$\tilde{V}_j=X_j^T\left[I-\Xa(\Xa^T\Xa)^{-1}\Xa^T\right]\frac{\e}{n}$,
then $V_j=\mu_j+\tilde{V}_j$. Using condition \eqref{IC}, we have
$|\mu_j|\leq (1-\eta)\lam$ for all $j \in S^c$, from which we obtain
that
$$\frac{1}{\lam} \max_{j\in S^c} |\tilde{V}_j|< \eta \Rightarrow \frac{\max_{j \in
S^c} |V_j|}{\lam}< 1.$$
By the Gaussian comparison result \eqref{GCR} stated in Lemma
\ref{GaussianComparison}, we have
$$P\left[\frac{1}{\lam}\max_{j\in
S^c}|\tilde{V}_j|\geq \eta\right]\leq
2(p-q)\exp\{-\frac{\lam^2\eta^2}{2\max_{j\in S^c}{E(\tilde
V_j^2)}}\}.$$ Since
$$E(\tilde{V}_j^2)=\frac{1}{n^2}X_j^TH [VAR(\e)] HX_j,$$
where $H=I-\Xa(\Xa^T\Xa)^{-1}\Xa^T$ which has maximum eigenvalue
equal to $1$, and $VAR(\e)$ is the variance-covariance matrix of
$\e$, which is a diagonal matrix with the $i$th diagonal element
equal to $\const\times|x_i^T\truebeta|$.

Since $|x_i^T\truebeta|\leq \sqrt{\|x_i(S)\|_2^2\|\truebeta\|_2^2}
\leq \max_i \|x_i(S)\|_2\|\truebeta\|_2$, an operator bound yields
$$E(\tilde{V}_j^2)\leq \frac{\const}{n^2}\max_i
\|x_i(S)\|_2\|\truebeta\|_2\|X_j\|_2^2=\frac{\const}{n}\max_i
\|x_i(S)\|_2\|\truebeta\|_2.$$ Therefore,
\begin{eqnarray*}
P\left[\frac{1}{\lam}\max_j|\tilde{V}_j|\geq \eta\right]&\leq&
2(p-q)\exp\left\{-\frac{n\lam^2\eta^2}{2\const\max_i
\|x_i(S)\|_2\|\truebeta\|_2}\right\}.
\end{eqnarray*}
So, we have
\begin{eqnarray*}
P\left[\frac{1}{\lam}\max_j|V_j|< 1\right]&\geq& 1-
P\left[\frac{1}{\lam}\max_j|\tilde{V}_j|\geq \eta\right]\\&\geq&
1-2(p-q)\exp\left\{-\frac{n\lam^2\eta^2}{2\const\|\truebeta\|_2\max_i
\|x_i(S)\|_2}\right\}.
\end{eqnarray*}

\textbf{Analysis of $\mathcal{M}(U):$}

$$\max_i |U_i|\leq \|(\frac{1}{n}\Xa^T\Xa)^{-1}\frac{1}{n}\Xa^T\e\|_{\infty}+\lam\|(\frac{1}{n}\Xa^T\Xa)^{-1}\overrightarrow{b}\|_{\infty}$$
Define $Z_i:=e_i^T(\frac{1}{n}\Xa^T\Xa)^{-1}\frac{1}{n}\Xa^T\e.$
Each $Z_i$ is a normal Gaussian with mean $0$ and variance
\begin{eqnarray*}
var(Z_i)&=&e_i^T(\frac{1}{n}\Xa^T\Xa)^{-1}\frac{1}{n}\Xa^T
[VAR(\e)]\frac{1}{n}\Xa(\frac{1}{n}\Xa^T\Xa)^{-1}e_i\\
&\leq& \frac{\const\|\truebeta\|_2\max_i\|x_i(S)\|_2}{n C_{\min}}.
\end{eqnarray*}
So, for any $t>0$, by \eqref{GCR}
$$P(\max_{i\in S} |Z_i|\ge t)\leq 2q
\exp\{-\frac{t^2n
C_{\min}}{2\const\|\truebeta\|_2\max_i\|x_i(S)\|_2}\},$$ by taking
$t=\frac{\lam\eta}{\sqrt{C_{\min}}}$, we have
$$P(\max_{i\in S} |Z_i| \ge \frac{\lam\eta}{\sqrt{C_{\min}}})\leq 2\q
\exp\left\{-\frac{n\lam^2\eta^2}{2\const\|\truebeta\|_2\max_i\|x_i(S)\|_2}\right\}.$$
Recall the definition of $\Psi(\X, \truebeta, \lam) =\lam\left[ \eta
\ (C_{\min})^{-1/2}+ \left\|\left(\frac{1}{n}\Xa^T\Xa \right)^{-1}
\overrightarrow{b} \right\|_\infty \right]$, we have
$$P(\max_i |U_i|\geq \Psi(\X,\truebeta,\lam))\leq 2\q
\exp\left\{-\frac{n\lam^2\eta^2}{2\const\|\truebeta\|_2\max_i\|x_i(S)\|_2}\right\}.$$
By condition $\minbeta>\Psi(\X,\truebeta,\lam)$, we have
$$P(\max_i |U_i|<\minbeta)\geq 1-2\q
\exp\left\{-\frac{n\lam^2\eta^2}{2\const\|\truebeta\|_2\max_i\|x_i(S)\|_2}\right\}.$$
At last, we have
$$P\left[\mathcal M(V) \& \  \mathcal M(U) \right]\geq
1-2\p
\exp\left\{-\frac{n\lam^2\eta^2}{2\const\|\truebeta\|_2\max_i\|x_i(S)\|_2}\right\}$$

\end{proof}

\subsection{Proof of Corollary 1}\label{proof:general:fixed}
\begin{proof}
Recall the definition of $\Gamma(\X, \truebeta, \const)$:
\[ \Gamma(\X, \truebeta, \const) = \frac{\eta^2\ SNR}{8\max_i\|x_i(S)\|_2(\eta \ C_{\min}^{-1/2}+\sqrt{q}C_{min}^{-1})^2\log(p+1)},\]
where $SNR = \frac{n[\minbeta]^2}{\const\|\truebeta\|_2}$.
So, \begin{eqnarray*} \frac{n\eta^2}{2\const
\|\truebeta\|_2\max_{i}\|x_i(S)\|_2}=\frac{ 4\Gamma(\X, \truebeta, \const)(\eta\ C_{\min}^{-1/2}+\sqrt{q} \
C_{\min}^{-1})^2\log(\p+1)}{[\minbeta]^2}
\end{eqnarray*}
By taking
\begin{equation*}
\lam=\frac{\minbeta }{ 2\left(\eta \
C_{\min}^{-1/2}+\sqrt{q} \ C_{\min}^{-1}\right)},
\end{equation*}
we have
\begin{eqnarray*}
\Psi(\X,\truebeta,\lam) &= &\lam\left[ \eta \ (C_{\min})^{-1/2}+ \left\|\left(\frac{1}{n}\Xa^T\Xa \right)^{-1} \overrightarrow{b} \right\|_\infty \right]\\
&\leq&  \lam\left[\eta\ C_{\min}^{-1/2}+\sqrt{q}C_{\min}^{-1}\right]\\
&=&\frac{\minbeta}{2}\\
&<&\minbeta,
\end{eqnarray*} and

\begin{eqnarray*}
\frac{n\lam^2\eta^2}{2\const
\|\truebeta\|_2\max_{i}\|x_i(S)\|_2}=\Gamma(\X, \truebeta, \const)\log(p+1).
\end{eqnarray*}
So, the probability bound in Theorem 1
greater than
$$1-2\exp\left\{ -\left({\Gamma(\X, \truebeta, \const)} -1 \right)\log(\p
+1)\right\},$$ which goes to one when $\Gamma(\X, \truebeta, \const,
\alpha)\rightarrow \infty$.

\end{proof}
\subsection{Proof of Theorem 2} \label{proof:NC}

\begin{proof}
First prove (b). Without loss of generality, assume for some $j \in
S^c$,
$X_{j}^T\Xa\left(\Xa^T\Xa\right)^{-1}\overrightarrow{b}=1+\zeta$,
then $V_{j}=\lam(1+\zeta)+\tilde{V}_j$, where
$\tilde{V}_j=-[\Xa\left(\Xa^T\Xa\right)^{-1}\Xa^T-I]\frac{\e}{n}$ is
a Gaussian random variable with mean $0$, so
$P(\tilde{V}_j>0)=\frac{1}{2}$. So, $P(V_j>\lam)\geq \frac{1}{2}$,
which implies that for any $\lam$, Condition \eqref{R1} (a necessary
condition) is violated with probability greater than $1/2$.

For claim (a).  Condition \eqref{R2},
$$sign\left(\betA+(\frac{1}{n}\Xa^T\Xa)^{-1}\left[\frac{1}{n}\Xa^T\e-\lam
sign(\betA)\right]\right)=sign(\betA)$$ is also a necessary
condition for sign consistency. Since
$\frac{1}{n}\Xa^T\Xa=I_{q\times q}$, \eqref{R2} becomes
$$sign\left(\betA+\left[\frac{1}{n}\Xa^T\e-\lam
sign(\betA)\right]\right)=sign(\betA),$$ which implies that
\begin{equation}
sign\left(\betA+\frac{1}{n}\Xa^T\e\right)=sign(\betA).
\label{samesign}
\end{equation}
Without loss of generality, assume for some $j \in S$,
$\truebetaj>0$. Then \eqref{samesign} implies $\truebetaj+Z_j>0,$
where $Z_j=e_j^T\frac{1}{n}\Xa^T\e$ is a Gaussian random variable
with mean $0$, and variance
\begin{eqnarray*}
var(Z_j)&=&e_j^T\frac{1}{n}\Xa^TVAR(\e)\frac{1}{n}\Xa e_j \\
&=& \frac{\const e_j^T\left[\Xa^T diag(|X\truebeta|)\Xa\right]
e_j}{n^2}\\
&=&\frac{\truebetaj^2}{c_{n,j}^2},
\end{eqnarray*} where the last equality uses the definition of
$c_{n,j}^2$ in Theorem 2.  To summarize,
\begin{eqnarray*}
P[\hbet =_s \truebeta]&\leq& P[\truebetaj+Z_j>0]\\
&=&P[Z_j>-\truebetaj]\\
&=&P[Z_j<\truebetaj]\\
& =& 1 - \int_{\truebetaj}^{\infty} \frac{1}{\sqrt{2\pi var(Z_j)}} \exp\{-\frac{x^2}{2 var(Z_j)}\} dx\\
& =& 1 - \int_{\truebetaj/\sqrt{var(Z_j)}}^{\infty} \frac{1}{\sqrt{2\pi }} \exp\{-\frac{x^2}{2}\} dx\\
&\leq& 1-\frac{1}{\sqrt{2\pi}}\int_{\truebetaj/\sqrt{var(Z_j)}}^{\infty} (\frac{x}{1+x}+\frac{1}{(1+x)^2})\exp\{-\frac{x^2}{2}\} dx\\
&=& 1-\frac{\exp\left\{-\frac{\truebetaj^2}{2var(Z_j)}\right\}}{\sqrt{2\pi}(1+\frac{\truebetaj}{\sqrt{var(Z_j)}})}\\
&=&
1-\frac{\exp\left\{-\frac{c_{n,j}^2}{2}\right\}}{\sqrt{2\pi}(1+{c_{n,j}})}.
\end{eqnarray*}

\end{proof}

\subsection{Proofs of Theorem 3}
\label{random:ensemble:proof}

To prove Theorem 3, we need some
preliminary results.

\begin{lemma} \label{Lemma:V}
Conditioned on $\Xa$ and $\e$, the random vector $V$ is Gaussian.
Its mean vector is upper bound as
\begin{equation}
\mid E[V|\e,\Xa]\mid\leq \lam(1-\eta)\textbf{1}.
\end{equation}
Moreover, its conditional covariance takes the form
\begin{equation}
cov[V|\e,\Xa]=M_n\Sigma_{2|1}=M_n[\SXbXb-\SXbXa(\SXaXa)^{-1}\SXaXb],
\end{equation}
where
\begin{equation}
M_n=\lam^2\overrightarrow{b}^T(\Xa^T\Xa)^{-1}\overrightarrow{b}+\frac{1}{n^2}\e^T[I-\Xa(\Xa^T\Xa)^{-1}\Xa^T]\e.
\end{equation}
\end{lemma}

\begin{lemma}\label{lemma:Mn}
Let
$M_1=\lam^2\overrightarrow{b}^T(\Xa^T\Xa)^{-1}\overrightarrow{b}$
and $M_2=\frac{1}{n^2}\e^T[I-\Xa(\Xa^T\Xa)^{-1}\Xa^T]\e$, then
$M_n=M_1+M_2$. We have
\begin{equation}
P\left[\frac{\lam^2q}{2n\tilde C_{\max}}\leq M_1\leq \frac{2\lam^2
q}{n \Tilde C_{\min}}\right]\geq 1-\exp\{-0.03n\},
\end{equation}
\begin{equation}
P\left[M_2\geq \frac{3\const \sqrt{\tilde
C_{\max}}\|\truebeta\|_2}{n}\right]\leq \frac{1}{n}.
\end{equation}
\end{lemma}

\begin{lemma}{\label{bound}}
\begin{equation}
P\left[\max_{i=1,\ldots n} \|x_i(S)\|_2^2\geq 2\Tilde C_{\max}
\max{(16q,4\log n)}\right]\leq \frac{1}{n}. \label{normofbeta}
\end{equation}
\end{lemma}

Proofs of these lemmas can be found in Appendix \ref{lemmaProofs}.
Now, we prove  Theorem 3.

 \textbf{ Analysis of $M(V)$:} Define the event $T=\{M_n\geq v^*\}$, where
$$v^*=\frac{2\lam^2 q
}{n\tilde C_{\min}}+\frac{3\const\sqrt{\Tilde
C_{\max}}\|\truebeta\|_2}{n}.$$ By Lemma \ref{lemma:Mn}, we have
$P[T]\leq \exp\{-0.03n\}+\frac{1}{n}$.

Let $\mu_j=E[V_j|\e,\Xa]$, $Z_j=V_j-\mu_j$, and $Z=(Z_j)_{j\in
S^c}$, then $E[Z|\Xa,\e]=0$ and $cov(Z|\Xa,
\e)=cov(V|\Xa,\e)=M_n\Sigma_{2|1}$.
\begin{eqnarray*}
\max_{j\in S^c} |V_j|&=&\max_{j\in S^c} |\mu_j+Z_j|\\
&\leq&\max_{j\in S^c }[|\mu_j|+|Z_j|]\\
&\leq&(1-\eta)\lam+\max_{j\in S^c }|Z_j|.
\end{eqnarray*} From this inequality, we have
\[ \{\max_{j\in S^c }|Z_j|< \eta \lam\}\subset \{\max_{j\in S^c
}|V_j|< \lam\}. \]
Define $\tilde{Z}$ to be a zero-mean Gaussian with covariance
$v^*\Sigma_{2|1}.$ Since
\begin{eqnarray*}
P\left[\max_{j\in S^c }|Z_j|\geq \eta \lam\mid T^c\right]&\leq&
\sum_{j\in S^c } P\left[|Z_j|> \eta \lam\mid T^c\right]\\
&\leq& (p-q)  \max_{j\in S^c } P\left[|\tilde{Z}_j|> \eta
\lam\right]\\
&\leq&2(p-q)\exp\{-\frac{\eta^2\lam^2}{2v^*\Tilde C_{\max}}\},
\end{eqnarray*}
we have \begin{eqnarray*} P[\max_{j\in S^c}|V_j|\geq \lam]&\leq&
P\left[\max_{j\in
S^c}|Z_j|\geq \lam\mid T^c\right]+P[T]\\
&\leq& 2(p-q)\exp\{-\frac{\eta^2\lam^2}{2v^*\Tilde
C_{\max}}\}+\exp\{-0.03n\}+\frac{1}{n}.
\end{eqnarray*}
This says that $$P[\mathcal M(V)]\geq
1-2(p-q)\exp\{-\frac{\eta^2\lam^2}{2v^*\Tilde
C_{\max}}\}-\exp\{-0.03n\}-\frac{1}{n}.$$

 \textbf{Analysis of $\mathcal M(U)$:} Now we analyze $\max_{j\in
S}|U_j|$.
$$\max_j |U_j| \leq \left\|(\frac{1}{n}\Xa^T\Xa)^{-1}\frac{1}{n}\Xa^T\e\right\|_\infty+  \lam\left\|(\frac{1}{n}\Xa^T\Xa)^{-1}\overrightarrow {b}\right\|_\infty.$$
Define $\Lambda_i(\cdot)$ to be the $i$th largest eigenvalue of a
matrix. Since
\begin{eqnarray*}
\lam\left\|(\frac{1}{n}\Xa^T\Xa)^{-1}\overrightarrow
{b}\right\|_\infty \leq \frac{\lam
\sqrt{q}}{\Lambda_{\min}(\frac{1}{n}\Xa^T\Xa)},
\end{eqnarray*}
by Equation \eqref{RMeigen} in Corollary \ref{minmax.eigen},
we have
\begin{equation*}
P\left[\lam\left\|(\frac{1}{n}\Xa^T\Xa)^{-1}\overrightarrow
{b}\right\|_\infty \leq \frac{2\lam \sqrt{q}}{\Tilde
C_{\min}}\right]\geq 1-2\exp(-0.03n).
\end{equation*}

Let \[ W_i= e_i^T(\frac{1}{n}\Xa^T\Xa)^{-1}\frac{1}{n}\Xa^T\epsilon,
\] then conditioned on $\Xa$, $W_i$ is a Gaussian random variable
with mean $0$, and variance
\begin{eqnarray*}
var(W_i|\Xa)&=&e_i^T(\frac{1}{n}\Xa^T\Xa)^{-1}\frac{1}{n}\Xa^T
[VAR(\e)]\frac{1}{n}\Xa(\frac{1}{n}\Xa^T\Xa)^{-1}e_i\\
&\leq& \frac{\const\|\truebeta\|_2\max_i\|x_i(S)\|_2}{n
\Lambda_{\min}(\frac{1}{n}\Xa^T\Xa)}.
\end{eqnarray*}

Using \eqref{RMeigen}
\begin{equation*}
P\left[\Lambda_i(\frac{1}{n}X^TX)\geq
\frac{1}{2}\Tilde C_{\min}\right]\geq 1-2\exp(-0.03n),
\end{equation*} and Lemma \ref{bound}, we have
\[
\frac{\const\|\truebeta\|_2\max_i\|x_i(S)\|_2}{n
\Lambda_{\min}(\frac{1}{n}\Xa^T\Xa)}\leq
\frac{2\const\|\truebeta\|_2\sqrt{2\Tilde C_{\max} \max{(16q,4\log
n)}}}{n\tilde C_{\min}}
\] with probability no less than $1-2\exp\{-0.03n\}-\frac{1}{n}.$

Define event
$$\mathcal{T}=\left\{\frac{\const\|\truebeta\|_2\max_i\|x_i(S)\|_2}{n
\Lambda_{\min}(\frac{1}{n}\Xa^T\Xa)}\leq
\frac{2\const\|\truebeta\|_2\sqrt{2\Tilde C_{\max} \max{(16q,4\log
n)}}}{n\tilde C_{\min}}\right\},$$ then $P(\mathcal{T})\geq
1-2\exp\{-0.03n\}-\frac{1}{n}.$ From the proof of Lemma
\ref{GaussianComparison}, for any $t>0$,
$$P(|W_i| > t \mid \Xa, \mathcal{T})\leq 2\exp(-\frac{t^2}{2 var(W_i \mid \Xa, \mathcal{T})}).$$ The above is also true if we replace $var(W_i \mid \Xa, \mathcal{T})$ with any upper bound. So, we have
$$P(|W_i|>t\mid \Xa, \mathcal{T})\leq 2\exp\left\{-\frac{t^2}{2\frac{2\const\|\truebeta\|_2\sqrt{2\Tilde C_{\max} \max{(16q,4\log
n)}}}{n\tilde C_{\min}}}\right\}.$$

So,
\begin{eqnarray*}
P(|W_i|>t)&\leq& P(|W_i|>t \ | \mathcal T) +P(\mathcal T^c)\\
&\leq &
2\exp\left\{-\frac{t^2}{2\frac{2\const\|\truebeta\|_2\sqrt{2\Tilde
C_{\max} \max{(16q,4\log n)}}}{n\tilde
C_{\min}}}\right\}+2\exp\{-0.03n\}+\frac{1}{n}.\end{eqnarray*}

By takeing $t=
A(n,\truebeta,\const):=\sqrt{\frac{4\const\|\truebeta\|_2\log
n\sqrt{2\max(16q,4\log n)}}{n\tilde C_{\min}}}$, we have
\begin{eqnarray*} P\left[\max_{i\in S} |W_i|>A(n,\truebeta,\const)\right]
&\leq& \frac{2\q}{n}+2q\exp\{-0.03n\}+\frac{q}{n}\\
&=& \frac{3\q}{n}+2q\exp\{-0.03n\}.
\end{eqnarray*}

Summarize,
\begin{eqnarray*}
P\left[\max_i |U_i| \geq A(n,\truebeta,\const)+\frac{2\lam\sqrt{q}}{\Tilde C_{\min}}\right]&\\
\leq\frac{3\q}{n}+2q\exp\{-0.03n\}+2\exp\{-0.03n\}&.
\end{eqnarray*}

At last, we have
$$P\left[\mathcal M(V) \ \& \ \mathcal M(U)\right]\leq 1-2(p-q)\exp\{-\frac{\eta^2\lam^2}{2v^*\Tilde C_{\max}}\}-(2q+3)\exp\{-0.03n\}-\frac{1+3q}{n}.$$

\subsection{Proofs of Corollary 3}
\label{proof:general:random}

\begin{proof}
By taking
$\lam=\frac{[\minbeta - A(n,\truebeta,\const)]\tilde C_{\min}}{4\sqrt{q}}$,
we have
\begin{eqnarray*}
\tilde \Psi(n, \truebeta, \lam, \const)&= &
A(n,\truebeta,\const)+\frac{2\lam\sqrt{q}}{\tilde C_{\min}}\\
& =& \frac{\minbeta + A(n,\truebeta,\const)}{2}\\
&<& \minbeta,
\end{eqnarray*}
where the last inequality uses the assumption that $\minbeta > A(n,\truebeta,\const)$.

\begin{eqnarray*}
\frac{\lam^2}{V^*(n, \truebeta, \lam, \const)} &=& \frac{\lam^2} {
\frac{2\lam^2 q
}{n\tilde C_{\min}}+\frac{3\const\sqrt{\Tilde C_{\max}}\|\truebeta\|_2}{n}}\\
&=&\frac{1}{\frac{2q}{n\tilde C_{\min}}+{\frac{3\const\sqrt{\Tilde C_{\max}}\|\truebeta\|_2}{n\lam^2}}}\\
&=&\frac{1}{\frac{2q}{n\tilde C_{\min}}+
\frac{48\const q\sqrt{\tilde C_{\max}}\|\truebeta\|_2}{n[\minbeta - A(n,\truebeta,\const)]^2 \tilde C_{\min}^2}}.
\end{eqnarray*}
By the definition of $\tilde \Gamma(n,\truebeta,\const)$, we have that
\begin{equation*}
\frac{\lam^2\eta^2}{2V^*(n, \truebeta, \lam, \const) \tilde C_{\max}} = \log(p-q+1)\tilde \Gamma(n,\truebeta,\const),
\end{equation*}
so the probability
bound in Theorem 3 now becomes,
\begin{eqnarray*}
&&1-2\exp\left\{-\frac{\lam^2\eta^2}{2V^*(n, \truebeta, \lam,
\const)\Tilde C_{\max}}+\log(\p-\q)\right\}-(2q+3)\exp\{-cn\}-\frac{1+3q}{n}\\
&&=1-2\exp\left\{-\log(p-q+1)\tilde \Gamma(n,\truebeta,\const)+\log(\p-\q)\right\}-(2q+3)\exp\{-cn\}\\
&&-\frac{1+3q}{n}\\
&&\geq 1-2\exp\left\{-\log(p-q+1)[\tilde \Gamma(n,\truebeta,\const)-1]\right\}-(2q+3)\exp\{-cn\}-\frac{1+3q}{n}
\end{eqnarray*}
If Condition \eqref{qnp} holds, then $\tilde \Gamma(n,\truebeta,\const, \alpha) \rightarrow \infty$ which  guarantees
$P[\hbet =_s \truebeta ]\rightarrow 1$.
\end{proof}

\subsection{Proof of Theorem
4}\label{proof:NC:random}
\begin{proof}
Without loss of generality, assume
$$e_j^T\SXbXa(\SXaXa)^{-1}sign(\betA) = 1+\zeta,$$
for some $j \in S^c$ and $\zeta \ge 0$. Since $E[V|\Xa,\e]=\lam
\SXbXa(\SXaXa)^{-1}sign(\betA)$, $V_j$, conditioned on $\Xa$ and
$\e$,
 is a Gaussian random variable
with mean $\lam(1+\zeta)$. So $P[V_j >
\lam(1+\zeta)|\Xa,\e]=\frac{1}{2}$, which implies $P[V_j>
\lam|\Xa,\e]\geq \frac{1}{2}$.  Then we have $P(V_j> \lam)\geq
\frac{1}{2}$. So for any $\lam$,
$$P[ \hbet =_s \truebeta ]\leq P[\max_j V_j \leq \lam] \leq \frac{1}{2}.$$
\end{proof}

\subsection{Proofs of Lemma \ref{Lemma:V} -- Lemma
\ref{bound}} \label{lemmaProofs}\mbox{}

\vspace{10 pt}

\textbf{Proof of Lemma \ref{Lemma:V}}
\begin{proof}
Conditioned on $\Xa$ and $\e$, the only random component in $V_j$ is
the column in  the column vector $X_j$, $j\in S^{c}$. We know that
$(\Xb|\Xa,\e)\sim (\Xb|\Xa)$ is Gaussian with mean and covariance
\begin{eqnarray}
E[\Xb^T|\Xa,\e]&=&\SXbXa(\SXaXa)^{-1}\Xa^T,\\
var(\Xb|\Xa)&=& \Sigma_{2|1}=\SXbXb-\SXbXa(\SXaXa)^{-1}\SXaXb.
\end{eqnarray}
Consequently, we have,
\begin{eqnarray*}
&&|E[V|\Xa,\e]|\\
&=&\left|\SXbXa(\SXaXa)^{-1}\Xa^T\left\{\Xa(\Xa^T\Xa)^{-1}\lam\overrightarrow{b}\right. \right.\\
&& \left.\left.-\left[\Xa(\Xa^T\Xa)^{-1}\Xa^T-I\right]\frac{\e}{n}\right\}\right|\\
&=&|\SXbXa(\SXaXa)^{-1}\lam\overrightarrow{b}|\\
&\leq&\lam(1-\eta)\textbf{1},
\end{eqnarray*}
where the last inequality uses Condition \eqref{IC2}.

Now, we compute the elements of the conditional covariance
\[
cov(V_j,V_k|\e, \Xa).
\]

Let $\vec{\alpha}=
\Xa(\Xa^T\Xa)^{-1}\lam\overrightarrow{b}-\left[\Xa(\Xa^T\Xa)^{-1}\Xa^T-I)\right]\frac{\e}{n}$,
then $V_j=X_j^T\vec{\alpha}$. So we have
\[
cov(V_j,V_k|\e, \Xa)=\vec{\alpha}^T
cov(X_j^T,X_k^T|\e,\Xa)\vec{\alpha}=\left[var(\Xb|\Xa)\right]_{jk}\vec{\alpha}^T\vec{\alpha}.
\]
Consequently,
\[
cov(V|\e,\Xa)=\vec{\alpha}^T\vec{\alpha} \
var(\Xb|\Xa)=\vec{\alpha}^T\vec{\alpha}\Sigma_{2|1}=\vec{\alpha}^T\vec{\alpha}[\SXbXb-\SXbXa(\SXaXa)^{-1}\SXaXb].
\]
By careful calculation, we have $\vec{\alpha}^T\vec{\alpha}=M_n$.
\end{proof}

\textbf{Proof of Lemma \ref{lemma:Mn}}

\begin{proof}
Recall that $M_1=\lam
\overrightarrow{b}^T(\Xa^T\Xa)^{-1}\overrightarrow{b} $. So,
$$\frac{\lam q}{\Lambda_{\max}(\Xa^T\Xa)}\leq M_1\leq
\frac{\lam q}{\Lambda_{\min}(\Xa^T\Xa)}.$$ From \eqref{RMeigen}
we have,
$$P\left[\frac{\lam q}{2n\tilde C_{\max}}\leq M_1\leq
\frac{2\lam q}{n\tilde C_{\min}}\right]\geq 1-2\exp(-0.03n).$$

Define $\varrho=E[|Z|],$ where $Z\sim N(0,1)$, then for any random
variable $R\sim N(0,\sigma^2)$, $E[|R|]=\sigma \varrho.$ Since
$x_i^T\truebeta\sim N(0, \betA^T\SXaXa\betA)$, we have
\[
E[|x_i^T\truebeta|]=\sqrt{\betA^T\SXaXa\betA}\varrho.
\]
We know that $M_2\leq \frac{1}{n^2}\e^T\e.$ Since
$E[\e_i^2]=E[E[\e_i^2|\Xa]]=E[\const|x_i^T\beta|]=\const\sqrt{\betA^T\SXaXa\betA}\varrho$,
and
$E[\e_i^4]=E[E[\e_i^4|\Xa]]=3E[\sigma^4|x_i^T\beta|^2]=3\sigma^4{\betA^T\SXaXa\betA}$,
we have
\begin{eqnarray*}
&&P\left[\frac{\sum_i\e_i^2}{n^2}\geq
\frac{\const(\varrho+\sqrt{3-\varrho^2})\sqrt{\betA^T\SXaXa\betA}}{n}\right]\\
&=& P\left[\sum_i\e_i^2-n\const \varrho\sqrt{\betA^T\SXaXa\betA}\geq
n\const\sqrt{3-\varrho^2}\sqrt{\betA^T\SXaXa\betA}\right]\\
&\leq& \frac{n var(\e_i^2)}{n^2 \sigma^4 (3-\varrho^2)\betA^T\SXaXa\betA}\\
&=& \frac{3\sigma^4{\betA^T\SXaXa\betA}-\sigma^4
\betA^T\SXaXa\betA\varrho^2}{n
\sigma^4 (3-\varrho^2)\betA^T\SXaXa\betA}\\
&=& \frac{1}{n}
\end{eqnarray*}
So,
$$P\left[M_2\geq
\frac{\const(\varrho+\sqrt{3-\varrho^2})\sqrt{\betA^T\SXaXa\betA}}{n}\right]\leq
\frac{1}{n}.$$ While $\sqrt{\beta_{1}^T\SXaXa\betA}\leq
{\sqrt{\Tilde C_{\max} }\|\beta\|_2}$ and $\varrho=E(|Z|)\leq
\sqrt{E(|Z|^2)}=1$, where $Z$ is a standard normal random variable,
so
\[\frac{\const(\varrho+\sqrt{3-\varrho^2})\sqrt{\betA^T\SXaXa\betA}}{n}\leq
\frac{3\const\sqrt{\Tilde C_{\max}}\|\beta\|_2}{n}.
\]
Then we have
\[
P[M_2\geq \frac{3\const\sqrt{\Tilde C_{\max}}\|\beta\|_2}{n}]\leq
\frac{1}{n}.
\]
\end{proof}

\textbf{Proof of Lemma \ref{bound}}
\begin{proof}

By lemma \ref{lemma:ldchi2}, we have for any $t>q$,
\[
P\left[\max_{i=1,\ldots n}
\|\Sigma_{11}^{-\frac{1}{2}}x_i(S)\|_2^2\geq 2t\right]\leq
n\exp(-t\left[1-2\sqrt{\frac{q}{t}}\right]).
\] Take $t=\max{(16q, 4\log n)}$, we have
\[
\begin{array}{lll}
\exp(-t\left[1-2\sqrt{\frac{q}{t}}\right])&\leq&\exp(-t\left[1-2\sqrt{\frac{1}{16}}\right])\\
&=&\exp(-\frac{t}{2})\\
&\leq&\frac{1}{n^2},
\end{array}
\]
so,
\[ P\left[\max_{i=1,\ldots n}
\|\Sigma_{11}^{-\frac{1}{2}}x_i(S)\|_2^2\geq 2 \max{(16q,4\log
n)}\right]\leq \frac{1}{n}.
\]
Since $\|\Sigma_{11}^{-\frac{1}{2}}x_i(S)\|_2^2\geq \frac{1}{\Tilde
C_{\max}}\|x_i(S)\|_2^2,$ we have \begin{equation}
P\left[\max_{i=1,\ldots n} \|x_i(S)\|_2^2\geq 2\Tilde C_{\max}
\max{(16q,4\log n)}\right]\leq \frac{1}{n}. \label{normofbeta}
\end{equation}

\end{proof}

\section{Some Gaussian Comparison Results}
\label{appendix:GCR}
\begin{lemma}
\label{GaussianComparison} For any mean zero Gaussian random vector
$(X_1,\ldots,X_n)$, and $t>0$, we have
\begin{equation}
P( \max_{1\leq i\leq n} |X_i| \ge t)\leq 2n
\exp\{-\frac{t^2}{2\max_i E(X_i^2)}\} \label{GCR}
\end{equation}
\end{lemma}
\begin{proof}
Note that the generate function of $X_i$ is
$$E(e^{tX_i})=\exp\{\frac{E(X_i^2)t^2}{2}\}.$$ So, for any $t>0$,
$$P(X_i\ge x)=P(e^{tX_i}\ge e^{tx})\leq
\frac{E(e^{tX_i})}{e^{tx}}=\exp\{\frac{E(X_i^2)t^2}{2}-xt\},$$ by
taking $t=\frac{x}{E(X_i^2)}$, we have
$$P(X_i\ge x)\leq \exp\{-\frac{x^2}{2E(X_i^2)}\}.$$
So,
$$P(|X_i|\ge t)=2P(X_i\ge t)\leq 2\exp\{-\frac{t^2}{2E(X_i^2)}\}\leq 2\exp\{-\frac{t^2}{2\max_i E(X_i^2)}\}.$$
So,
$$P(\max_{1\leq i\leq n}|X_i|\ge t)\leq 2n \exp\{-\frac{t^2}{2\max_{i} E(X_i^2)}\}.$$

\end{proof}

\section{Large deviation for $\chi^2$ distribution} \label{LargeDev}
\begin{lemma}
\label{lemma:ldchi2} Let $Z_1,\ldots,Z_n$ be i.i.d.
$\chi^2$-variates with $q$ degrees of freedom. Then for all $t > q$,
we have
\begin{equation}
P\left[\max_{i=1,\ldots,n} Z_i>2t\right]\leq
n\exp(-t\left[1-2\sqrt{\frac{q}{t}}\right]).
\end{equation}
\end{lemma}
The proof of this lemma can be found in \cite{obozinski2008union}.
\section{Some useful random matrix results}
\label{RMR} In this appendix, we use some known concentration
inequalities for the extreme eigenvalues of Gaussian random matrices
 \citep{DavidsonS2001} to bound the
eigenvalues of a Gaussian random matrix. Although these results hold
more generally, our interest here is on scalings $(n, q)$ such that
$q/n\rightarrow 0$.

\begin{lemma}[\cite{DavidsonS2001}]\label{DS}
Let $\Gamma \in R^{n\times \q}$ be a random matrix whose entries are
i.i.d.\ from $N(0,1/n)$, $\q \leq n$. Let the singular values of
$\Gamma$ be $s_1(\Gamma)\geq \ldots \geq s_\q(\Gamma)$. Then
$$\max\left\{P\left[s_1(\Gamma)\geq 1 + \sqrt{\frac{q}{n}}+t\right], P\left[s_\q(\Gamma) \leq
1-\sqrt{\frac{q}{n}}-t\right]\right\}< \exp\{-nt^2/2\}.$$
\end{lemma}

Using Lemma \ref{DS}, we now have some useful results.
\begin{lemma}{\label{rmatr}} Let $U\in R^{n\times \q}$ be a random matrix with elements from the
standard normal distribution (i.e., $U_{ij}\sim N(0,1)$, i.i.d.)
Assume that $q/n\rightarrow 0$. Let the eigenvalues of
$\frac{1}{n}U^TU$ be $\Lambda_1(\frac{1}{n}U^TU)\geq \ldots\geq
\Lambda_\q(\frac{1}{n}U^TU)$. Then when
$n$ is big enough,

\begin{equation}P\left[\frac{1}{2} \leq \Lambda_i(\frac{1}{n}U^TU) \leq 2\right]\geq 1 - 2\exp(-0.03n).
\label{eigen:bound}
\end{equation}

\end{lemma}

\begin{proof}
Let $\Gamma = \frac{1}{\sqrt{n}} U$, then
$\Lambda_i(\frac{1}{n}U^TU) = s_i^2(\Gamma)$. By Lemma \ref{DS},
$$P\left[s_\q(\Gamma) \leq
1-\sqrt{\frac{q}{n}}-t\right]< \exp\{-nt^2/2\}, $$ by taking $t =
t_0 = 1- \frac{\sqrt{2}}{2} - 0.1$, we have
$$P\left[s_\q(\Gamma) \leq
\frac{\sqrt{2}}{2} + 0.1 -\sqrt{\frac{q}{n}}\right]<
\exp\{-nt_0^2/2\}.$$ Since $ q/ n \rightarrow 0$ by assumption, we
have when n is big enough, $\sqrt{q/n} < 0.1$, then
$$P\left[s_\q(\Gamma) <
\frac{\sqrt{2}}{2} \right]< \exp\{-nt_0^2/2\},$$ which implies that,
for any $i = 1, \ldots, q$,
$$P\left[\Lambda_i(\frac{1}{n}(U^TU)) <
\frac{1}{2}\right]< \exp\{-nt_0^2/2\}.$$ Followed the same
procedures, $$P\left[\Lambda_i(\frac{1}{n}(U^TU)) > 2\right]<
\exp\{-nt_1^2/2\},$$ for $t_1 = \sqrt{2}- 1.1$. Then inequality
\eqref{eigen:bound} holds immediately.
\end{proof}

\begin{corollary}{\label{minmax.eigen}}
Let $X\in R^{n\times \q}$ be a random matrix, of which, the rows are
i.i.d.\ from the normal distribution with mean 0 and covariance
$\Sigma.$ Assume that $0 < \Tilde C_{\min}\leq \Lambda_i(\Sigma)\leq
\Tilde C_{\max} < \infty$ and $\q/n\rightarrow 0$, then when $n$ is big enough,

\begin{equation}
P\left[\frac{1}{2}\Tilde C_{\min}\leq\Lambda_i(\frac{1}{n}X^TX)\leq
2\Tilde C_{\max}\right]\geq 1-2\exp(-0.03n).\label{RMeigen}
\end{equation}
\end{corollary}

\begin{proof}%
Let $U=X\Sigma^{-\frac{1}{2}}$, then  $U$
satisfies  the condition in Lemma \ref{rmatr}. Then
$$P\left[\frac{1}{2}\leq\Lambda_i(\frac{1}{n}U^TU)\leq
2\right]\geq 1-2\exp(-0.03n).$$ Since
$$\Tilde C_{\min}\Lambda_1(\frac{1}{n}U^TU)\leq \Lambda_i(\frac{1}{n}X^TX)\leq \Tilde C_{\max}\Lambda_q(\frac{1}{n}U^TU),$$
result \eqref{RMeigen} is obtained immediately.
\end{proof}

\vskip 0.2in

\noindent{\large\bf References}
\begin{description}
\bibitem[\protect\citeauthoryear{Candes and Tao}{Candes and
  Tao}{2005}]{CandesT2005}
{  Candes, E.} {  and} {  Tao, T.} 2005.
\newblock Decoding by linear programming.
\newblock {\em IEEE Trans. Info Theory,\/}~{\em 51(12)}, 4203 -- 4215.

\bibitem[\protect\citeauthoryear{Candes and Tao}{Candes and
  Tao}{2007}]{candes2007dss}
{  Candes, E.} {  and} {  Tao, T.} 2007.
\newblock {The Dantzig selector: statistical estimation when p is much larger
  than n}.
\newblock {\em Annals of Statistics\/}~{\em 35,\/}~6, 2313--2351.

\bibitem[\protect\citeauthoryear{Chatterjee.}{Chatterjee.}{2005}]{Chatterjee20%
05}
{  Chatterjee., S.} 2005.
\newblock An error bound in the sudakov-fernique inequality.
\newblock {\em Technical report, UC Berkeley. arXiv:math.PR/0510424\/}.

\bibitem[\protect\citeauthoryear{Chen, Donoho, and Saunders.}{Chen
  et~al\mbox{.}}{1998}]{ChenDS1998}
{  Chen, S.}, {  Donoho, D.~L.}, {  and} {  Saunders., M.~A.} 1998.
\newblock Atomic decomposition by basis pursuit.
\newblock {\em J. Sci. Computing\/}~{\em 20(1)}, 33--61.

\bibitem[\protect\citeauthoryear{Davidson and Szarek}{Davidson and
  Szarek}{2001}]{DavidsonS2001}
{  Davidson, K.~R.} {  and} {  Szarek, S.~J.} 2001.
\newblock {\em Local operator theory, random matrices, and Banach spaces. In
  Handbook of Banach Spaces, volume 1, pages 317-366.}
\newblock Elsevier, Amsterdan, NL.

\bibitem[\protect\citeauthoryear{Donoho}{Donoho}{2004}]{Donoho2004}
{  Donoho}. 2004.
\newblock For most large undetermined system of linear equations the minimal
  l1-norm near-solution is also the sparsest solution.
\newblock {\em Technical report, Statistics Department, Stanford University\/}.

\bibitem[\protect\citeauthoryear{Donoho, Elad, and Temlyakov}{Donoho
  et~al\mbox{.}}{2006}]{DonohoET2006}
{  Donoho, D.}, {  Elad, M.}, {  and} {  Temlyakov, V.~M.} 2006.
\newblock Stable recovery of sparse overcomplete representations in the
  presence of noise.
\newblock {\em IEEE Trans. Info Theory\/}~{\em 52(1)}, 6--18.

\bibitem[\protect\citeauthoryear{Donoho and Huo.}{Donoho and
  Huo.}{2001}]{DonohoH2001}
{  Donoho, D.} {  and} {  Huo., X.} 2001.
\newblock Uncertainty principles and ideal atomic decomposition.
\newblock {\em IEEE Trans. Info Theory\/}~{\em 47(7)}, 2845 -- 2862.

\bibitem[\protect\citeauthoryear{E.~Candes and Tao.}{E.~Candes and
  Tao.}{2004}]{CandesRT2004}
{  E.~Candes, J.~R.} {  and} {  Tao., T.} 2004.
\newblock Robust uncertainty principles: exact signal reconstruction from
  highly incomplete frequency information.
\newblock {\em Technical report, Applied and Computational Mathematics,
  Caltech\/}.

\bibitem[\protect\citeauthoryear{Efron, Hastie, Johnstone, and
  Tibshirani}{Efron et~al\mbox{.}}{2004}]{efron2004least}
{  Efron, B.}, {  Hastie, T.}, {  Johnstone, I.}, {  and} {  Tibshirani, R.}
  2004.
\newblock {Least angle regression}.
\newblock {\em Annals of statistics\/}, 407--451.

\bibitem[\protect\citeauthoryear{Elad and Bruckstein.}{Elad and
  Bruckstein.}{2002}]{EladB2002}
{  Elad, M.} {  and} {  Bruckstein., A.~M.} 2002.
\newblock A generalized uncertainty principle and sparse representation in
  pairs of bases.
\newblock {\em IEEE Trans. Info Theory,\/}~{\em 48(9)}, 2558 -- 2567.

\bibitem[\protect\citeauthoryear{Elad and Bruckstein.}{Elad and
  Bruckstein.}{2003}]{FeuerN2003}
{  Elad, M.} {  and} {  Bruckstein., A.~M.} 2003.
\newblock On sparse representation in pairs of bases.
\newblock {\em IEEE Trans. Info Theory,\/}~{\em 49(6)}, 1579 -- 1581.

\bibitem[\protect\citeauthoryear{Fessler}{Fessler}{2000}]{fessler2000statistic%
al}
{  Fessler, J.} 2000.
\newblock {Statistical image reconstruction methods for transmission
  tomography}.
\newblock {\em Handbook of Medical Imaging\/}~{\em 2}, 1--70.

\bibitem[\protect\citeauthoryear{Freedman}{Freedman}{2005}]{freedman2005}
{  Freedman, D.} 2005.
\newblock {\em {Statistical models: Theory and practice}}.
\newblock Cambridge University Press.

\bibitem[\protect\citeauthoryear{Fuchs}{Fuchs}{2005}]{fuchs2005recovery}
{  Fuchs, J.} 2005.
\newblock {Recovery of exact sparse representations in the presence of bounded
  noise}.
\newblock {\em IEEE Transactions on Information Theory\/}~{\em 51,\/}~10,
  3601--3608.

\bibitem[\protect\citeauthoryear{Knight and Fu}{Knight and
  Fu}{2000}]{KnightF2000}
{  Knight, K.} {  and} {  Fu, W.~J.} 2000.
\newblock Asymptotics for lasso-type estimators.
\newblock {\em Annals of Statistics,\/}~{\em 28}, 1356 -- 1378.

\bibitem[\protect\citeauthoryear{Ledoux and Talagrand}{Ledoux and
  Talagrand}{1991}]{Ledoux1991}
{  Ledoux, M.} {  and} {  Talagrand, M.} 1991.
\newblock {\em Probability in Banach Spaces: Isoperimetry and Processes.}
\newblock New York, Springer-Verlag.

\bibitem[\protect\citeauthoryear{Lustig, Donoho, Santos, and Pauly}{Lustig
  et~al\mbox{.}}{2008}]{lustig2008compressed}
{  Lustig, M.}, {  Donoho, D.}, {  Santos, J.}, {  and} {  Pauly, J.} 2008.
\newblock {Compressed sensing MRI}.
\newblock {\em IEEE Signal Processing Magazine\/}~{\em 25,\/}~2, 72--82.

\bibitem[\protect\citeauthoryear{Massart.}{Massart.}{2003}]{Massart2003}
{  Massart., P.} 2003.
\newblock {\em Concentration Inequalties and Model Selection.}
\newblock Ecole d'Et$\acute{e}$ de Probabilit$\acute{e}$s,Saint- Flour.
  Springer, New York.

\bibitem[\protect\citeauthoryear{Meinshausen and Buhlmann}{Meinshausen and
  Buhlmann}{2006}]{MeinshausenB2006}
{  Meinshausen, N.} {  and} {  Buhlmann, P.} 2006.
\newblock High-dimensional graphs and variable selection with the lasso.
\newblock {\em Annals of Statistics\/}~{\em 34(3)}, 1436--1462.

\bibitem[\protect\citeauthoryear{Obozinski, Wainwright, Jordan,
  et~al\mbox{.}}{Obozinski et~al\mbox{.}}{2008}]{obozinski2008union}
{  Obozinski, G.}, {  Wainwright, M.}, {  Jordan, M.} 2008.
\newblock {Union support recovery in high-dimensional multivariate regression}.
\newblock {\em stat\/}~{\em 1050}, 5.

\bibitem[\protect\citeauthoryear{Osborne, Presnell, and Turlach}{Osborne
  et~al\mbox{.}}{2000}]{Osborne2000}
{  Osborne, M.~R.}, {  Presnell, B.}, {  and} {  Turlach, B.~A.} 2000.
\newblock On the lasso and its dual.
\newblock {\em Journal of Computational and Graphical Statistics\/}~{\em 9(2)},
  319--37.

\bibitem[\protect\citeauthoryear{Rosset}{Rosset}{2004}]{Rosset2004}
{  Rosset, S.} 2004.
\newblock Tracking curved regularized optimization solution paths.
\newblock {\em NIPS\/}.

\bibitem[\protect\citeauthoryear{Tibshirani}{Tibshirani}{1996}]{Tibshirani1996}
{  Tibshirani, R.} 1996.
\newblock Regression shrinkage and selection via the lasso.
\newblock {\em Journal of the Royal Statistical Society, Series B\/}~{\em
  58(1)}, 267--288.

\bibitem[\protect\citeauthoryear{Tropp.}{Tropp.}{2004}]{Tropp2004}
{  Tropp., J.} 2004.
\newblock Greed is good: algorithmic results for sparse approximation.
\newblock {\em IEEE Trans. Info Theory,\/}~{\em 50(10)}, 2231 -- 2242.

\bibitem[\protect\citeauthoryear{Tropp}{Tropp}{2006}]{Tropp2006}
{  Tropp, J.} 2006.
\newblock Just relax: Convex programming methods for identifying sparse signals
  in noise.
\newblock {\em IEEE Trans. Info Theory\/}~{\em 52(3)}, 1030 -- 1051.

\bibitem[\protect\citeauthoryear{Vardi, Shepp, and Kaufman}{Vardi
  et~al\mbox{.}}{1985}]{vardi1985statistical}
{  Vardi, Y.}, {  Shepp, L.}, {  and} {  Kaufman, L.} 1985.
\newblock {A statistical model for positron emission tomography}.
\newblock {\em Journal of the American Statistical Association\/}~{\em
  80,\/}~389, 8--20.

\bibitem[\protect\citeauthoryear{Wainwright}{Wainwright}{2009}]{Wainwright2009}
{  Wainwright, M.} 2009.
\newblock Sharp thresholds for high-dimensional and noisy recovery of sparsity.
\newblock {\em IEEE Transactions on Information Theory\/}~{\em To appear}.

\bibitem[\protect\citeauthoryear{Zhao and Yu}{Zhao and Yu}{2006}]{ZhaoY2006}
{  Zhao, P.} {  and} {  Yu, B.} 2006.
\newblock On model selection consistency of lasso.
\newblock {\em The Journal of Machine Learning Research\/}~{\em 7}, 2541--2563.

\bibitem[\protect\citeauthoryear{Zhao and Yu}{Zhao and Yu}{2007}]{Zhao2007}
{  Zhao, P.} {  and} {  Yu, B.} 2007.
\newblock Stagewise lasso.
\newblock {\em The Journal of Machine Learning Research\/}~{\em 8}, 2701--2726.

\bibitem[\protect\citeauthoryear{Zou}{Zou}{2006}]{zou2006adaptive}
{  Zou, H.} 2006.
\newblock {The adaptive lasso and its oracle properties}.
\newblock {\em Journal of the American Statistical Association\/}~{\em
  101,\/}~476, 1418--1429.

\end{description}

%%%%%%%%%%%%%%%%%%%%%%%%%%%%%%%%%%%%%%%%%%%%%%%%%%%%%%%%%%%%%%%%%%%%%%%%%%%%%%%%%%%%%%%%%%%%%%%%%%%%%%%%%%%%%%%%%%%%%%%%%%%%

\vskip .65cm
\noindent
Department of Statistics,
       University of California,
       Berkeley, CA 94720, USA
\vskip 2pt
\noindent
E-mail: (jzjia@stat.berkeley.edu)
\vskip 2pt

\noindent
\vskip .65cm
\noindent
Department of Statistics,
       University of California,
       Berkeley, CA 94720, USA
\vskip 2pt
\noindent
E-mail: (karlrohe@stat.berkeley.edu)
\vskip 2pt

\noindent
\vskip .65cm
\noindent
Department of Statistics,
and Department of Electrical Engineering and Computer Sciences,
       University of California,
       Berkeley, CA 94720, USA
\vskip 2pt
\noindent
E-mail: (binyu@stat.berkeley.edu)
\vskip .3cm

%%%%%%%%%%%%%%%%%%%%%%%%%%%%%%%%%%%%%%%%%%%%%%%%%%%%%%%%%%%%%%%%%%%%%%%%%%%%%%%%%%%%%%%%%%%%%%%%%%%%%%%%%%%%%%%%%%%%%%%%%%%%
%%%%%%%%%%%%%%%%%%%%%%%%%%%%%%%%%%%%%%%%%%%%%%%%%%%%%%%%%%%%%%%%%%%%%%%%%%%%%%%%%%%%%%%%%%%%%%%%%%%%%%%%%%%%%%%%%%%%%%%%%%%%

\end{document}